\documentclass[authoryear,preprint,review,11pt,3p,times]{elsarticle}

\usepackage{amssymb}
\usepackage{amsthm}
\newtheorem{lemma}{Lemma}
\usepackage{amsmath}
\usepackage{bbm}

\usepackage[authoryear]{natbib}
\usepackage{url}
\usepackage{xcolor}
\usepackage{comment}

\usepackage{booktabs}
\usepackage{multirow}

\usepackage{tikz}
\usepackage{pgfplots}
\pgfplotsset{compat=1.18} 

\usepackage{float}

\journal{European Journal of Operational Research}

\begin{document}

\begin{frontmatter}



\title{Classical and Deep Reinforcement Learning Inventory Control Policies for Pharmaceutical Supply Chains with Perishability and Non-Stationarity}


\author[1,2]{Francesco Stranieri\corref{cor1}}
\author[3]{Chaaben Kouki}
\author[4]{Willem van Jaarsveld}
\author[2]{Fabio Stella}

\cortext[cor1]{Corresponding author. E-mail address: \url{francesco.stranieri@polito.it}.}

\affiliation[1]{organization={University of Milano-Bicocca},
                city={Milan},
                postcode={20126},
                country={Italy}}

\affiliation[2]{organization={Polytechnic of Turin},
                city={Turin},
                postcode={10129},
                country={Italy}}
                
\affiliation[3]{organization={ESSCA School of Management},
                city={Angers},
                postcode={49000},
                country={France}}
                
\affiliation[4]{organization={Eindhoven University of Technology},
                city={Eindhoven},
                postcode={5612},
                country={Netherlands}}

\begin{abstract}
We study inventory control policies for pharmaceutical supply chains, addressing challenges such as perishability, yield uncertainty, and non-stationary demand, combined with batching constraints, lead times, and lost sales. Collaborating with Bristol-Myers Squibb (BMS), we develop a realistic case study incorporating these factors and benchmark three policies--order-up-to (OUT), projected inventory level (PIL), and deep reinforcement learning (DRL) using the proximal policy optimization (PPO) algorithm--against a BMS baseline based on human expertise. We derive and validate bounds-based procedures for optimizing OUT and PIL policy parameters and propose a methodology for estimating projected inventory levels, which are also integrated into the DRL policy with demand forecasts to improve decision-making under non-stationarity. Compared to a human-driven policy, which avoids lost sales through higher holding costs, all three implemented policies achieve lower average costs but exhibit greater cost variability. While PIL demonstrates robust and consistent performance, OUT struggles under high lost sales costs, and PPO excels in complex and variable scenarios but requires significant computational effort. The findings suggest that while DRL shows potential, it does not outperform classical policies in all numerical experiments, highlighting 1) the need to integrate diverse policies to manage pharmaceutical challenges effectively, based on the current state-of-the-art, and 2) that practical problems in this domain seem to lack a single policy class that yields universally acceptable performance.
\end{abstract}




\end{frontmatter}


\section{Introduction}
Pharmaceuticals are closely tied to patient health, highlighting the critical importance of accurate inventory control policies in supply chains. However, managing medical product inventories is challenging due to the interaction of multiple \textit{factors}, including random yields, perishability, batching constraints, non-stationary demand caused by product life cycles, and lost sales \citep[see also][]{song2023research}. Although the impact of these factors in isolation is reasonably well understood \citep[see, e.g.,][]{Broekmeulen2009,Sonntag2016,Sonntag2018,Gorria2022}, what is their relative importance, and how do they interact in real-world pharmaceutical supply chains? Moreover, while several policies in the literature address one of these challenges individually, is it clear how to adapt them to realistic supply chains that feature multiple overlapping challenges? What performance can be expected from different types of policies?

To address these questions, we developed a realistic case study in close collaboration with a senior supply chain manager from Bristol-Myers Squibb (BMS), a global manufacturer and distributor of medical products that faces typical pharmaceutical inventory challenges. The case study focuses on a specific product for which production yield uncertainty, product lifetime, and batching considerations are quantified using company data and expert input. BMS faces unique demand patterns for each product, necessitating an inventory management process that involves demand forecasting and timely ordering from manufacturing facilities. The case study incorporates a baseline policy based on expert human planners at BMS, who rely on forecasting models to predict demand uncertainty and fine-tune inventory levels to prevent and mitigate the risks of excess stock, product expiration, and lost sales. To explore the impact of demand uncertainty and non-stationarity, we used \textit{synthetic data} from BMS regarding product life cycles to develop a realistic demand process covering the product's lifetime for 20 years and a variant with a shorter lifetime of 5 years. For parameters that are challenging to estimate accurately, \textit{sensitivity analyses} were conducted over a broad range of values centered around company-provided data. This approach enabled the creation of a set of experiments that effectively represent the complexities of managing medical product inventories, including considerations of batching, product lifetime, and yield uncertainty. The resulting case study provides a robust foundation for answering key questions about inventory policies in pharmaceutical supply chains.

In detail, we contribute to the literature by adapting three general-purpose policies: \textit{i)} an order up-to-level (OUT) policy, with a bounds-based search procedure to determine appropriate safety stock levels for non-stationary demand; \textit{ii)} a projected inventory level (PIL) policy, which has been shown to perform well in settings involving lost sales and perishability, supported by a bounds-based procedure to optimize policy parameters; and \textit{iii)} a deep reinforcement learning (DRL) approach, implemented via the proximal policy optimization (PPO) algorithm \citep{schulman2017proximal}, which has been promoted as a general-purpose solution for inventory management, though its practical applications remain limited \citep{Boute2022}. To successfully train the DRL algorithm, we introduce a novel method for designing features in the presence of non-stationarity, lost sales, and product expiration. Departing from the standard practice of using the inventory vector directly as input to the neural network \citep[see, e.g.,][]{Oroojlooyjadid2022,Gijsbrechts2022,DeMoor2022,Kaynov2024,Stranieri2024_IJPE,Stranieri2024}, we estimate future projected inventory levels for each age category based on the current state and use these estimates as input. Furthermore, we incorporate specific time-dependent demand features consistent with demand forecasts used by human planners, enabling PPO to account for the current life-cycle phase accurately.

In our experiments, we benchmark the performance of these three policies against a BMS baseline derived from human planners' actions. This comparison yields new and valuable insights into inventory management for medical products:
\begin{itemize}
    \item OUT and PIL policies can be readily applied to \textit{pharmaceutical supply chains} using the bounds-based optimization procedures developed in this paper, while PPO can be successfully implemented based on the designed features. All three policies demonstrate competitive performance.
    \item While there are considerable performance differences among the three solution approaches, none consistently outperforms the others. In fact, each policy (OUT, PIL, and PPO) is outperformed by more than 10\% in specific, realistic experiments by one of the other policies. The OUT policy, in particular, can perform poorly when lost sales costs are high, with performance gaps exceeding 100\%. These findings imply:
    \textit{a)} Companies are advised to explore multiple policy types when addressing pharmaceutical inventory challenges and should \textit{avoid relying solely on the OUT policy}.
    \textit{b)} Despite decades of inventory research, practical inventory problems in this domain appear to \textit{lack a single policy class that delivers universally acceptable performance}, let alone an interpretable, near-optimal policy.
    \item Although DRL algorithms are often promoted as general-purpose solutions for complex, realistic inventory problems \citep{Boute2022,Vanvuchelen2020}, we find PPO to be competitive but not consistently superior to classical policies across all experiments.
    \item The proposed OUT, PIL, and PPO policies can significantly reduce total company costs by maintaining lower inventory levels. However, although they outperform the human-driven policy in terms of average cost, they exhibit higher \textit{performance variability}, which may render them less robust in certain scenarios. Our findings suggest that human planners focus on maintaining high service levels to avoid lost sales and ensure patient health, often achieving this by increasing safety stock levels. This practice reduces the risk of unsatisfied demand but results in higher holding costs. Enhancing the proposed policies with safeguards to address significant demand uncertainty could improve BMS results by balancing cost and service level requirements.
\end{itemize}

The remainder of this paper is structured as follows: Section \ref{sec:section2} presents a comprehensive review of recent literature on classical policies and DRL. Section \ref{sec:section3} outlines the mathematical formulation of the pharmaceutical supply chain developed in collaboration with BMS, while Section \ref{sec:section4} details and formalizes the OUT and PIL policies and describes the implementation of the PPO algorithm. Section \ref{sec:section5} presents the results of numerical experiments evaluating the performance of the implemented policies. Finally, Section \ref{sec:section6} concludes the paper.

\section{Related Work} \label{sec:section2}
Our case study is motivated by a realistic inventory problem arising in pharmaceutical supply chains, which involves multiple factors, including perishability and yield uncertainty. While \textit{random yield} is a well-studied problem in the literature when considered independently \citep{Sonntag2016,Sonntag2018,Berling2022}, its interaction with \textit{perishability}, as observed in our real-world case study, has received limited attention. To position our work within these research streams, we review studies that address these factors in-depth, along with a concise overview of OUT and PIL policies and the application of DRL algorithms in inventory management.

\subsection{Classical Policies}
Perishable inventory systems have been extensively studied since the seminal work of \cite{van1963inventory}. A key challenge in managing such systems is tracking the age categories of stock to determine optimal order quantities. This process significantly increases the problem's dimensionality, making it computationally complex. Commonly used approaches, such as the OUT policy, are suboptimal because order quantities must consider the age distribution of existing stock \citep{song2023research}. \cite{Nahmias1975} contributed by analyzing optimal policies for perishable systems with zero lead times, emphasizing the trade-off between holding and expiration costs. Managing perishable products remains challenging because the need to account for all age categories makes exact solutions through dynamic programming infeasible, especially when positive lead times are involved or product lifetimes exceed two periods. Consequently, much subsequent research has focused on developing heuristic or approximation-based approaches \citep{Chao2015}.

Early research by \cite{Nandakumar1993} pioneers near-myopic heuristics to address the computational challenges of dynamic programming in managing fixed-lifetime perishable products. The estimated withdrawal and aging (EWA) policy proposed by \cite{Broekmeulen2009} represents a significant advancement. This policy accounts for positive lead times and varying demand patterns while estimating the quantity of expired items during lead times. By explicitly considering the impact of product expiration, the EWA policy significantly improves upon the standard OUT policy, resulting in better cost performance. Further exploration of the EWA policy by \cite{Haijema2019} highlights the benefits of incorporating estimated expired quantities into base-stock policies. Their findings suggest that adding expired estimates can significantly improve performance compared to traditional base-stock policies that do not include such adjustments. Additionally, \cite{Gorria2022} examines supply chains for platelet concentrates, providing analytical approximations for key performance indicators such as expected on-hand inventory, order size, lost sales, and expired items. This research considers several complexities, including non-stationary demand with weekly variations and adjustable safety stock levels.

The \textit{asymptotic optimality} of policy classes has received increasing attention in recent studies. \cite{Huh2009} demonstrate that for non-perishable products in lost-sales settings, the base-stock (OUT) policy converges to optimality as penalty costs increase. For perishable products, \cite{Bu2023_4400} investigate the effectiveness of the OUT policy in lost-sales settings with no lead times and in backorder settings with positive lead times, establishing its asymptotic optimality as penalty or expiration costs grow. The PIL policy has primarily been studied in the context of asymptotic optimality. \cite{vanJaarsveld2024} analyze it for the standard lost-sales setting, providing results for long lead times and high penalty costs. Additionally, \cite{Bu2023_4638265} examine perishable products with positive lead times, demonstrating that the PIL policy is optimal for single-period product lifetimes under bounded demand. For scenarios with unbounded demand and high penalty costs, the PIL policy still achieves optimality. \cite{Goldberg2021} provides further insights into the concept of asymptotic optimality.

The main thrust of our paper is to evaluate classical policies and DRL in a pharmaceutical supply chain. Our \textit{adaptations} of the classical OUT and PIL policies for this purpose build on the reviewed literature as follows. First, in contrast to the EWA policy examined by \cite{Broekmeulen2009} and \cite{Gorria2022}, which approximates expired stock using the mean demand value, our adaptation of the PIL policy also considers lost sales and uses the true distribution of demand to estimate expired stock more accurately. Additionally, in the context of the OUT policy with lost sales, perishable products, and positive lead times, we define both lower and upper bounds, heuristically addressing cases of non-stationary demand not analyzed by \cite{Bu2023_4400}. While their approach is practical in settings with backlogged demand or zero lead times, our case study focuses on settings involving lost sales and positive lead times. Lastly, we propose a bounds-based procedure to optimize the parameter of the PIL policy in the BMS case study, building on the work of \cite{Bu2023_4638265} and establishing an effective procedure for accurately estimating the projected inventory level.

\subsection{Deep Reinforcement Learning}
\cite{Boute2022} provide a comprehensive roadmap detailing the potential improvements that DRL can offer to inventory systems and policies. \cite{Oroojlooyjadid2022} extend the deep Q-network (DQN) algorithm to address the beer game problem, enabling the DQN algorithm to learn a near-optimal policy while other entities follow a base-stock policy. \cite{Gijsbrechts2022} implement and fine-tune the asynchronous advantage actor-critic (A3C) algorithm for multi-echelon systems, finding that A3C performs comparably to state-of-the-art heuristics and approximate dynamic programming algorithms.

To the best of our knowledge, no studies have assessed the ability of DRL to solve problems that combine multiple factors such as non-stationary demand, yield uncertainty, and perishability--challenges commonly arising in real-world cases. Recently, \citep{vanderHaar2024} explored the use of DRL in industrial spare parts management. However, that case features neither perishability nor yield uncertainty. In the following, we briefly review DRL applications that incorporate at least one of these factors.

Non-stationarity is challenging to incorporate into classical policies, and DRL algorithms may offer significant advantages in such scenarios. \cite{Stranieri2024} evaluate the performance of DRL algorithms and classical policies in multi-echelon systems with stochastic and seasonal demand, showing that the PPO algorithm performs better than other policies in a wide range of experiments. Similarly, \cite{Dehaybe2024} demonstrate that DRL, when supported by demand forecasting, can effectively learn non-stationary policies in supply chains characterized by fixed costs, lead times, and the presence of both backorders and lost sales. Their findings reveal that DRL can match or even surpass the performance of dynamic programming-based heuristics in certain scenarios. \citet{vanHezewijk2024} propose scalable DRL algorithms for a lot-sizing problem involving multiple products, where individual items exhibit non-stationarity. Additionally, \citet{vanDijck2024} apply DRL to a complex capacitated assembly system, showing that DRL is especially valuable in handling non-stationary demand.

A notable research gap exists in applying DRL in the domain of perishable products with fixed lifetimes. \cite{temizoz2020deep} develop DRL algorithms tailored to highly stochastic problems arising in operations management, testing their solution approach on three benchmark problems, one of which involves perishable products. \cite{DeMoor2022} conduct a comparative analysis of the DQN algorithm against classical policies and other heuristics, consistently demonstrating that DQN outperforms alternatives across various experiments. Similarly, \cite{Ahmadi2022} investigate the use of DRL algorithms for perishable products, finding that these algorithms significantly reduce the likelihood of product stockouts and expirations, thereby enhancing service levels.

Our paper explores the application of DRL in inventory systems involving key factors such as perishability, yield uncertainty, and non-stationary demand--among others--that are critical in pharmaceutical supply chains and are rarely addressed together. Unlike prior studies, which typically examine these factors in isolation, we focus on integrating them into a realistic case study. By leveraging future projected inventory levels as features, our approach bridges theoretical insights and practical implementation, offering a novel method for managing complex inventory problems.

\section{Case Study Description} \label{sec:section3}
The pharmaceutical supply chain analyzed in this paper is modeled based on a \textit{real-world case study}. One of the authors collaborated with a senior supply chain manager at BMS to obtain a validated model of the inventory system and access a realistic demand-generating process for a specific perishable pharmaceutical drug. The supply chain manager also provided key operational parameter values, such as yield uncertainty, product lifetime, and lead time, as discussed in detail below. Using company data and expert input, we derived a validated set of parameters that produces a realistic model for assessing state-of-the-art inventory policies in the context of challenges typically arising in pharmaceutical supply chains. This section describes the base case, while Section \ref{sec:section5} explores variants of this case to gain deeper insights into the results.

The modeled supply chain environment is structured as a finite-horizon, periodic-review system over $T$ timesteps ($t = 1, \ldots, T$), directly linking a third-party factory with a storage warehouse, as represented in Figure \ref{fig:sc1}. We consider a single product type produced by the factory. Product batches may experience up to 10\% of items lost due to yield uncertainty during production. In our base case, BMS assumes that yield loss is uniformly distributed between $0$ and $10\%$ of each order. This assumption captures the inherent variability in production, enabling us to explore a broader range of yield rate values.

Supply chains commonly operate with planned lead times to provide the factory with sufficient time for production planning \citep[see, e.g.,][]{Tang2007}. At BMS, a positive lead time of $L = 12$ timesteps is assumed, with items delivered in batches of $Q = 20$ units and a maximum of $n = 6$ batches per shipment. The company employs a dynamic batch ordering cost $K(q_t)$, which varies with the batch size to account for economies of scale, along with a static unit ordering cost $\hat{c}$ proportional to the number of items ordered.

Other parameters are similarly based on discussions with expert human planners at BMS. Upon receipt at the warehouse, each item has a product lifetime of $m = 12$ timesteps, after which an expiration cost of $\hat{w} = 3.0$ per unit is incurred. The warehouse is assumed to have unlimited storage capacity, with a holding cost of $\hat{h} = 1.0$ per unit per timestep. A lost sales cost of $\hat{b} = 100$ per unit is applied as a penalty when demand exceeds on-hand inventory. Without loss of generality, we assume that any remaining items at the end of the horizon are salvaged at a value of $\hat{c}$ per unit.

\begin{figure}[ht!]
    \centering
        \includegraphics[scale=0.65]{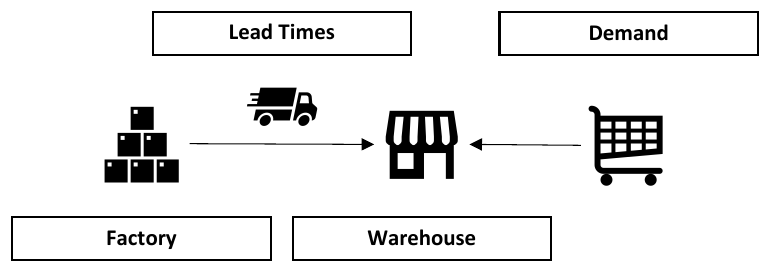}
        \caption{Representation of the supply chain environment.}
        \label{fig:sc1} 
\end{figure}

Demand across timesteps follows a non-stationary process. For each timestep $t$, let $D_t$ represent the one-period demand, with $\mathbb{E}[D_t] < +\infty$. The demand $D_t$ at each timestep $t$ is modeled as $D_t = d_t + \xi_t$, where $d_t$ represents deterministic values provided by BMS that define the demand forecast, and $\xi_t$ denotes random variables with a mean of zero, representing the forecast error. The non-stationarity arises because the standard deviation of $\xi_t$, denoted as $\sigma_t$, varies across timesteps. BMS provides forecast errors as percentages of the mean demand, which are used to compute the time-dependent standard deviation $\sigma_t$ for each timestep. The distribution of $\xi_t$ is not necessarily independent and identically distributed (i.i.d.), further capturing the non-stationarity of the forecast error.

\subsection{Order of Events}
As illustrated in Figure \ref{fig:dynamics}, the order of events in the supply chain environment and their associated costs for each timestep $t$ are defined as follows:
\begin{enumerate}
    \item The order $q_{t-L}$ placed at timestep $t-L$ arrives at the warehouse.
    \item A new order $q_{t}$ is placed at the factory, incurring ordering costs.
    \item Demand $D_t$ is satisfied from the on-hand inventory in the warehouse. Any unsatisfied demand is tracked as lost sales.
    \item Costs for timestep $t$--including ordering, lost sales, expiration, and holding costs--are computed. Expired items are removed from the on-hand inventory at the end of the timestep, incurring expiration costs. The remaining items with usable lifetimes are moved forward to the next timestep, resulting in holding costs.
\end{enumerate}

\begin{figure}[ht!]
    \centering
    \resizebox{10cm}{3cm}{ 
    \begin{tikzpicture}
    \draw[thick, black] (0,0) node[above=3pt, thick, black] {$t$} node[above=4pt] {};
    \draw[thick, black] (0,-0.1) -- (0,0.1);
    \draw[thick, ->, black] (-1,0) -- (11,0);
    \draw[thick, ->, black] (2,0.7) -- (2,0);
    \draw[thick, ->, black] (4,-0.7) -- (4,0);
    \draw[thick, ->, black] (6,0.7) -- (6,0);
    \draw[thick, ->, black] (8,-0.7) -- (8,0);
    \draw[thick, black] (2,0.7) node[above=2pt, align=center, thick] {order $q_{t-L}$ \\ arrives} node[above=3pt] {};
    \draw[thick, black] (4,-0.7) node[below=2pt, align=center, thick] {order $q_{t}$ \\ placed} node[above=3pt] {};
    \draw[thick, black] (6,0.7) node[above=2pt, align=center, thick] {demand $D_{t}$ \\ satisfied} node[above=3pt] {};
    \draw[thick, black] (8,-0.7) node[below=2pt, align=center, thick] {costs $C_{t}^{\pi}$\\ computed} node[above=3pt] {};
    \draw[thick, black] (10,0) node[above=3pt,thick, black] {$t+1$} node[above=3pt] {};
    \draw[thick, black] (10,-0.1) -- (10,0.1);
    \end{tikzpicture}
    }
    \caption{Order of events in the supply chain environment.}
    \label{fig:dynamics}
\end{figure}
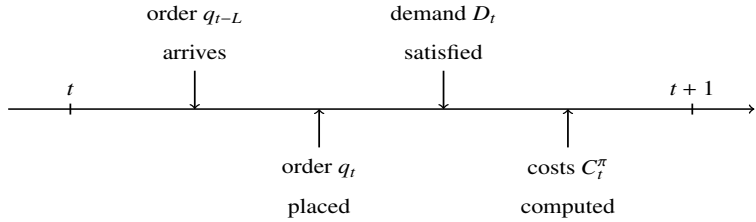

Let $\mathbf{x}_t := (x_{t,1}, x_{t,2}, \ldots, x_{t,m}, \ldots, x_{t,m+L-1})$ be the vector representing the inventory in transit and on hand, where $x_{t,i}$ denotes the quantity in transit or on hand with a remaining lifetime of $i$, for $i = 1, \ldots, m+L-1$. Let $Z_t$ be a random variable representing the fractional yield for the order arriving at timestep $t$ with realization $z_t$. The maximum production loss, denoted by $\hat{z}$, is set at 10\%, resulting in $Z_t = 1 - \mathcal{U}(0, 1) \times \hat{z}$, where $\mathcal{U}(0, 1)$ is a random variable uniformly distributed on the interval $[0, 1]$. This yield effectively captures the variability in the drug production process, ensuring it fluctuates between 90\% and 100\%, as indicated by BMS. Consequently, the order arriving at timestep $t+1$, denoted by $x_{t+1, m}$, is updated based on the yield rate $Z_t$ and the previous order $x_{t, m+1}$.

The full \textit{inventory vector} $\mathbf{x}_t$ is then updated as follows:
\begin{equation}\label{Eq:Dynamic}
x_{t+1,i} =
\begin{cases}
q_t, \, i = m+L-1, \, x_{t,i+1}, \, i = m+1, \ldots, m+L-2, \\
\left(Z_t x_{t,i+1} - \left(D_t - \sum_{j=1}^{i} x_{t,j}\right)^{+}\right)^{+}, \, i = m, \\
\left(x_{t,i+1} - \left(D_t - \sum_{j=1}^{i} x_{t,j}\right)^{+}\right)^{+}, \, i = 1, \ldots, m-1.
\end{cases}
\end{equation}
This update implies that for the inventory at position $i = m+L-1$, the vector is updated according to the current order quantity $q_t$ placed at timestep $t$. For intermediate positions, from $i = m+1$ to $i = m+L-2$, representing the pipeline of inventory in transit, the vector is shifted by one position. For positions from $i = 1$ to $i = m$, the vector is updated to satisfy the demand $D_t$ using the on-hand inventory, following a FIFO issuing policy from the oldest to the most recent items. Inventory with a lifetime of $i = m$ (i.e., items that have just arrived) directly incorporates the yield rate $Z_t$. Here, $(\cdot)^{+}$ denotes $\max(0, \cdot)$.

Let $X_t$ and $Y_t$ denote the random variables representing the total inventory level at timestep $t$ before order arrival and after order arrival but before demand satisfaction, respectively. Their realizations are defined as $x_t = \sum_{i=1}^{m-1} x_{t,i}$ and $y_t = x_t + z_t q_{t-L}$, and the expired quantity at timestep $t$ is given by $O_t = (x_{t,1} - D_t)^+$.

The \textit{total cost} over a finite episode horizon $T$ for a given policy $\pi \in \Pi$, where a policy maps states to actions and $\Pi$ represents the set of feasible policies, can be expressed as:
\begin{equation}\label{Eq:tot cost}
C^{\pi} = \mathbb{E}\left[ \sum_{t=1}^{T} C_{t}^{\pi} - \hat{c} \sum_{i=1}^{m-1} x_{T+1,i} \right],
\end{equation}
where $C_{t}^{\pi} = K(q_t) + \hat{c} q_t + \hat{h} (Y_{t} - D_t)^{+} + \hat{b} (D_t - Y_{t})^{+} + \hat{w} O_t$ and $K(q_t) = \{0, 5, 8, 9, 10\}$, for $n_t = 0, 1, 2, 3, \{4, 5, 6\}$, respectively, with $n_t = \lceil \frac{q_t}{Q} \rceil$ and $\lceil \cdot \rceil$ represents the ceiling function.
The first term in $C_{t}^{\pi}$ represents the batch ordering cost, conditional on the batch size, while the second term is the unit ordering cost proportional to the number of items. The subsequent terms denote the holding, lost sales, and expiration costs at timestep $t$, respectively.

Given our objective to determine an appropriate order quantity $q_t$ at each timestep $t$, while minimizing the total cost $C^{\pi}$ across the episode horizon $T$, we assume $q_t = 0$ for $t = T - L + 1, \ldots, T$, as these orders will not arrive before the episode ends. Additionally, at the start of each episode, the inventory level is assumed to be empty, i.e., $x_{1,i} = 0$ for $i = 1, \ldots, L - 1$. Table \ref{tab:notation} summarizes the notation used throughout the paper.

\begin{table}[ht!]
\centering
\footnotesize
\caption{Notation for the supply chain environment with their explanations (and units of measure).}
\label{tab:notation}
\renewcommand{\arraystretch}{0.5} 
\begin{tabular}{llll}
\toprule
Parameter & Explanation & Parameter & Explanation \\ \midrule
$K(q_t)$ & Batch Ordering Cost (per batch) & $T$ & Episode Horizon (timesteps) \\
$\hat{c}$ & Unit Ordering Cost (per unit) & $L$ & Lead Time (timesteps) \\
$\hat{w}$ & Expiration Cost (per unit) & $m$ & Product Lifetime (timesteps) \\
$\hat{b}$ & Lost Sales Cost (per unit per timestep) & $\hat{z}$ & Production Yield (percentage per batch) \\
$\hat{h}$ & Holding Cost (per unit per timestep) & $D_{t}$ & Demand (units) \\
\bottomrule
\end{tabular}
\end{table}

\subsection{Cost Transformation}
Following a methodology similar to that described in \cite{Chao2015}, we transform the total cost $C^{\pi}$ into an \textit{equivalent expression} that excludes the unit ordering cost $\hat{c}$. The following lemma facilitates this transformation.

\begin{lemma}\label{Lemma1}
The total cost defined in Equation \ref{Eq:tot cost} can be expressed as:
\begin{equation}\label{Eq:cost trans}
C^\pi = \mathbb{E}\left[\sum_{t=1}^{T} K(q_t) + h(Y_{t} - D_t)^{+} + b(D_t - Y_{t})^{+} + wO_t\right] + \hat{c}\sum_{t=L+1}^{T}  \mathbb{E}[D_{t}],
\end{equation}
where $h = \hat{h}$, $b = \hat{b} - \hat{c}$, and $w = \hat{w} + \hat{c}$.
\end{lemma}

\begin{proof}
The proof is provided in \ref{appendix:1}.
\end{proof}

From Lemma \ref{Lemma1}, it follows that finding an optimal policy $\pi$ minimizing $C^\pi$ is equivalent to finding the same optimal policy for $C^\pi - \hat{c} \sum_{t=L+1}^{T} \mathbb{E}[D_{t}]$. Consequently, the optimal policy can be defined as $OPT = \inf_{\pi \in \Pi} C^\pi$. In the subsequent sections, we focus on optimizing the transformed cost $C^\pi$.

\subsection{Dynamic Programming Formulation}
Under the lost sales assumption, determining the optimal order quantity $q_t$ using dynamic programming depends on the inventory vector $\mathbf{x}_t$. Let $V_t(\mathbf{x}_t)$ denote the cost-to-go function from timestep $t$ to $T$:
\begin{equation}\label{eq:dp}
V_{t}(\mathbf{x}_t) = \min \mathbb{E}\left[ K(q_t) + h(Y_{t} - D_t)^{+} + b(D_t - Y_{t})^{+} + wO_t + V_{t+1}(\mathbf{x}_{t+1}) \right].
\end{equation}

Solving this dynamic programming model poses significant challenges due to the \textit{curse of dimensionality}, primarily caused by the positive lead time $L$ and the product lifetime $m$. For instance, \cite{Ding2023} demonstrated that even with a lead time of zero, determining the optimal order quantity for products with a lifetime of at least five timesteps requires substantial computational effort. Introducing a lead time of 12 timesteps, combined with a product lifetime of 12 timesteps--as in our real-world case study--further exacerbates the complexity, making it computationally intractable to find an exact solution within a reasonable timeframe. This challenge persists even for non-perishable products, as solving Equation \ref{eq:dp} is further complicated by the lost sales assumption, as highlighted in prior studies \citep{Levi2008,arts2015base}.

As expressed in Lemma \ref{Lemma1}, the unit ordering cost $\hat{c}$ is incorporated into the holding, penalty, and expiration costs, allowing us to focus on solving the transformed total cost $C^\pi$. However, the batch ordering cost represented by $K(q_t)$ introduces additional complexities. Developing effective policies for perishable products becomes particularly difficult when such costs are present \citep{Dehaybe2024}. As a result, the proposed policies must be carefully evaluated to ensure their effectiveness in addressing these challenges.

\section{Inventory Policies} \label{sec:section4}
This section analyzes two primary \textit{types of inventory policies} implemented under the specific assumptions and challenges of our supply chain environment: classical (analytic) policies and DRL policies. For the classical policies, we define the widely known and adopted OUT policy (Section \ref{sec:bs}) and describe the baseline policy currently used at BMS, which refines the OUT policy by incorporating expected expiration in the spirit of the EWA policy (Section \ref{sec:bmspol}). Additionally, we discuss the PIL policy applied in our case study (Section \ref{sec:pil}). Lastly, for the DRL policy, we introduce the PPO algorithm (Section \ref{sec:ppo}). Our analysis outlines the fundamental principles of each policy type, investigates specific theoretical properties, and discusses their practical implementation.

\subsection{OUT Policy} \label{sec:bs}
To implement the OUT policy in our non-stationary case, we propose fixing the safety stock instead of the order-up-to level. This approach ensures that the order-up-to level adjusts dynamically based on the expected demand during lead times, allowing it to increase or decrease in response to fluctuations in demand forecasts. Specifically, the proposed OUT policy is characterized by a single parameter, $s$, representing the \textit{safety stock}. Denoted by $\pi_s$, this policy involves ordering to the \textit{order-up-to level} $S_t=s + \sum_{j=t}^{t+L} d_j$ at each timestep $t$. Accordingly, the order quantity $q_t^{\pi_s}$ at timestep $t$ is calculated as:
\begin{equation}
q_{t}^{\pi_s}=\left(S_t-\sum_{i=1}^{m+L-1}x_{t,i}\right)^+,
\end{equation}
and the number of batches is determined as: $n_t^{\pi_s} = \left\lceil \frac{q_t^{\pi_s}}{Q}\right\rceil$. 

The expected total cost over the episode horizon $T$ is given by:
\begin{equation}
C^{\pi_s}=\mathbb{E}\left[\sum_{t=1}^{T}(h(Y_t-D_t)^{+}+b(D_t-Y_t)^{+}+wO_t)\right].
\end{equation}
This formulation implies that the order-up-to level equals $s + \sum_{j=t}^{t+L} d_j$, which may vary dynamically at each timestep $t$ to account for non-stationary demand. Here, $\sum_{i=1}^{m+L-1} x_{t,i}$ represents the inventory level at timestep $t$ prior to ordering. Given the positive lead time $L$, the expected total cost can also be expressed as:
\begin{equation}
C^{\pi_s} = \sum_{t=1}^{T-L} \left(h \mathbb{E}\left[\left(Y_{t+L}-D_{t+L}\right)^+\right] + b \mathbb{E}\left[\left(D_{t+L}-Y_{t+L}\right)^+\right] + w \mathbb{E}\left[\left(x_{t+L,1}-D_{t+L}\right)^+\right]\right).
\end{equation}

Our goal is to determine the optimal safety stock $s$ for the OUT policy under lost sales. To achieve this, we propose a bounds-based search procedure that leverages approximative assumptions and simplifications to establish bounds and then employs a bounds-based procedure to optimize policy parameters. Specifically, we assume that the random variables representing forecast error, $\xi_t$, are i.i.d. and follow a normal distribution with a mean of zero and a standard deviation $\sigma$. Under this assumption, the cumulative forecast error over the lead time is defined as $\mathcal{D}_{L+1} = \sum_{j=t}^{t+L} \xi_j$. Two simplifications are also made to facilitate the derivation of bounds. First, batch ordering costs, $K(q_t)$, are omitted. Prior studies, such as \cite{Nahmias1978} and \cite{Zhang2016}, analyzed fixed ordering costs in perishable systems with fixed lifetimes but limited their focus to scenarios with zero lead time. In pharmaceutical supply chains, transportation costs are typically shared across multiple products, making fixed ordering costs negligible. Second, yield uncertainty is excluded from deriving theoretical bounds to maintain tractability. Given the complexities of deriving closed-form expressions for costs, we adopt \textit{Monte Carlo simulation} to determine the optimal safety stock for the OUT policy. In particular, we generate 2000 simulated episodes, each covering the entire episode horizon, and set the safety stock to minimize the average total cost across these simulated episodes. The effectiveness of this bounds-based procedure will be validated in Section \ref{sec:section5}.

\subsubsection{OUT Policy Lower Bound}
The following lemma establishes a lower bound for the cost component of the OUT policy.

\begin{lemma}\label{Lemma3}
Under the OUT policy with lost sales, for each demand sample path and every timestep $t \geq L+1$, the following conditions hold:
 \begin{equation}
\begin{cases}
\mathbb{E}\left[O_t\right]\geq \frac{\mathbb{E}\left[\left(s-\mathcal{D}_{m+L}\right)^+\right]}{m+L}\\
\mathbb{E}\left[\left(\sum _{i=1}^m x_{t,i}-D_t\right)^+\right]\geq\mathbb{E}\left[\left(s-\mathcal{D}_{L+1}\right)^+\right]-\frac{L\mathbb{E}\left[\left(s-\mathcal{D}_{m+L}\right)^+\right]}{m+L}  \\
\mathbb{E}\left[\left(D_t-\sum _{i=1}^m x_{t,i}\right)^+\right]\geq\frac{\mathbb{E}\left[\left(\mathcal{D}_{L+1}-s\right)^+\right]}{L+1}
\end{cases}.
 \end{equation}
\end{lemma}

\begin{proof}
The proof is provided in \ref{appendix:3}.
\end{proof}

This lemma enables us to derive a lower bound, denoted by $LB$, on the total cost:
\begin{equation}
C^{\pi_s} \geq LB(s) = T \left(h \mathbb{E}\left[\left(s-\mathcal{D}_{L+1}\right)^+\right] + \frac{b}{L+1} \mathbb{E}\left[\left(\mathcal{D}_{L+1}-s\right)^+\right] + \frac{w-hL}{m+L} \mathbb{E}\left[\left(s-\mathcal{D}_{m+L}\right)^+\right] \right).
\end{equation}

\subsubsection{OUT Policy Upper Bound}
We next derive an upper bound, $UB$, under the assumption that items are non-perishable ($O_t = 0$ for $m = +\infty$ and any $t > L$):
\begin{equation}
\begin{aligned}
\left(\sum_{i=1}^m x_{t,i}-D_t\right)^+ &\leq \left(\sum_{j=t-L}^{t-1} l_j\right) + \left(s-\sum_{j=t-L}^t \xi_j\right)^+, \\
\left(D_t-\sum_{i=1}^m x_{t,i}\right)^+ &\leq \left(\sum_{j=t-L}^t \xi_j-s\right)^+.
\end{aligned}
\end{equation}

Taking the expectation of the above inequalities yields the upper bound:
\begin{equation} \label{eq:ub_out}
UB(s)=T\left(h \mathbb{E}\left[\left(s-\mathcal{D}_{L+1}\right)^+\right]+(b+hL)\mathbb{E}\left[\left(\mathcal{D}_{L+1}-s\right)^+\right]\right).
\end{equation}
This bound follows directly from the assumption that $\xi_t$ is stationary and i.i.d., which implies that $\sum_{j=t-L}^t \xi_j$ and $\sum_{j=t}^{t+L} \xi_j$ share the same distribution.

We proceed by determining the solution for the upper bound, denoted as $s^*_{UB}$, which corresponds to the $\frac{b+hL}{b+(L+1)h}$ quantile of the forecast error distribution:
\begin{equation}
 s^*_{UB}= \inf\left\{s: \mathbb{P}(\mathcal{D}_{L+1}\leq s)\geq \frac{b+hL}{b+(L+1)h}\right\}.
\end{equation}
The upper bound is expected to act as a cap on the order-up-to level of the optimal perishable OUT policy under lost sales. This expectation is based on the formal demonstration provided by \cite{Nandakumar1993} for scenarios with zero lead time and further supported by findings for positive lead times in non-perishable systems by \cite{Huh2009}. However, providing formal proof of this conjecture in the context of perishable products remains a challenging open problem.

Finally, given the complexities of deriving closed-form expressions for costs, we adopt \textit{Monte Carlo simulation} to determine the optimal safety stock for the OUT policy. Specifically, we generate 2000 simulated episodes, each covering the entire episode horizon, and set the safety stock to minimize the average total cost across these simulated episodes.

\subsection{BMS Baseline Policy} \label{sec:bmspol}
We next discuss the BMS baseline policy, which is derived from the expertise of human planners. Similar to the OUT policy, this human-driven policy relies on a time-dependent target level, defined as:
\begin{equation}\label{BMS-1}
    \tilde{S}_t = \tilde{s} + \sum_{j=t}^{t+L} d_{j},
\end{equation}
where $\tilde{s}$ represents the safety stock.

Accordingly, the BMS baseline orders the quantity $n_{t}^{\pi_{\tilde{s}}} = \left\lceil\frac{q_t^{\pi_{\tilde{s}}}}{Q}\right\rceil$ at each timestep $t$, where:
\begin{equation}\label{PIL-3}
    q_t^{\pi_{\tilde{s}}} = \left(\tilde{S}_t - \left(\sum_{i=1}^{m+L-1}x_{t,i} - \sum_{j=t}^{t+L-1} \hat{O}_{j} |\mathbf{x}_t\right)\right)^+.
\end{equation}
This expression does not account for lost sales occurring during the $L$ timesteps and requires the computation of $\sum_{j=t}^{t+L-1} \hat{O}_{j}$, which represents a \textit{heuristic estimate of the expired quantity} from timestep $t$ to $t+L-1$, based on the assumption that demand equals its expected value.

In particular, let $\hat{O}_{[t:s]}$ denote the cumulative estimated expired quantity from timestep $t$ to $s \geq t$. According to Equation 3 in \cite{Chao2018}, this can be computed as:
\begin{equation}\label{PIL-4}
    \hat{O}_{[t:s]}(\mathbf{x}_t)=\sum_{j=t}^{s}\hat{O}_{j}= \max\left\{\sum_{i=1}^{s-t+1}x_{t,i}-d_{[t:s]}, \hat{O}_{[t:s-1]}(\mathbf{x}_t)\right\},
\end{equation}
with the convention that $\hat{O}_{[t:-1]}(\mathbf{x}_t)\equiv0$. This approach is related to the EWA policy proposed by \cite{Broekmeulen2009}, but BMS applies a specific procedure for determining the safety stock.

\subsubsection{BMS Baseline Policy Safety Stock}
The safety stock for the BMS baseline policy is determined using a standard safety stock factor $k_1$ to achieve a 99\% service level, assuming the forecast error $\xi_t$ follows a normal distribution. Additionally, BMS empirically introduces a second safety stock factor $k_2$ to address the uncertainty in production yield. Based on the mean squared error (MSE) dynamically calculated at each timestep $t$, the safety stock is derived as:
\begin{equation}\label{PIL-2}
    \tilde{s} = k_1\sqrt{(L+1) \cdot MSE} + k_2.
\end{equation}

Note that the MSE, and consequently the safety stock $\tilde{s}$, depends on the quality of the demand forecast. Unlike the OUT policy, which determines the safety stock through cost-based optimization, the BMS baseline employs an internal \textit{fixed rule} for safety stock determination.

\subsection{PIL Policy} \label{sec:pil}
The foundation of the PIL policy has been explored by \cite{vanJaarsveld2024} and further developed by \cite{Bu2023_4638265} to address perishable products with fixed lifetimes, demonstrating its effectiveness compared to constant-order and OUT policies. However, neither study has examined its application under non-stationary demand. This paper extends the analysis of \cite{Bu2023_4638265} to investigate such scenarios. In the original formulation of the PIL policy, \cite{vanJaarsveld2024} propose placing an order to raise the \textit{projected inventory level} to a desired target. In our extension, we adopt a slightly modified approach to account for the complexities introduced by non-stationary demand.

Furthermore, both the BMS baseline and EWA policies account for expirations during lead times when determining the size of orders to be placed but ignore lost sales in the computation of Equation \ref{PIL-3}, treating excess demand as backorders. These policies also assume that forecast errors follow a normal distribution and rely solely on mean demand to estimate expirations, disregarding any associated uncertainty. To address these limitations, we propose incorporating the expected inventory level at timestep $t+L$ into the PIL policy.

Let $x_{t+L} = \sum_{i=1}^{m-1} x_{t+L,i}$ denote the inventory level before ordering at timestep $t+L$, and let $\pi_{u}$ represent the PIL policy, which involves ordering at each timestep $t$ up to:
\begin{equation}\label{PIL-1}
    U_t = u + d_{t+L},
\end{equation}
where $u$ represents the safety stock employed by the PIL policy. By the \textit{material conservation law}, the inventory level can be expressed as:
\begin{equation}\label{PIL-5}
x_{t+L}=\sum_{i=1}^{m-1}x_{t+L,i}=\sum_{i=1}^{m+L-1}x_{t,i}-\sum_{j=t}^{t+L-1}D_{j}+\sum_{j=t}^{t+L-1}l_{j}-\sum_{j=t}^{t+L-1}O_{j},
\end{equation}
where $l_{t}=(D_t-\sum_{i=1}^{m}x_{t,i})^+$.

According to Lemma 1 of \cite{Bu2023_4638265}, it is always possible to order a quantity that brings the expected inventory level at timestep $t+L$ to $U_{t}$, denoted as $\mathbb{E}\left[\sum_{i=1}^m x_{t+L,i}\right] = U_{t}$. By Equation \ref{PIL-5}, the order quantity can be expressed as:
\begin{equation}\label{PIL-6}
q_{t}^{\pi_u} = \left(U_{t}-\mathbb{E}\left[\sum _{i=1}^{m-1} x_{t+L,i}|\mathbf{x}_t\right]\right)^+ = \left(u+\sum_{j=t}^{t+L}d_{j}-\sum_{i=1}^{m+L-1}x_{t,i}+\mathbb{E}\left[\sum_{j=t}^{t+L-1}O_{j}-l_{j}\right]\right)^+.
\end{equation}

Therefore, the optimal order quantity, accounting for cumulative lost sales over $L$ timesteps, is given by:
\begin{equation}\label{PIL-8}
n_t^{\pi_u}=\left\lceil\frac{q_{t}^{\pi_u}}{Q}\right\rceil.
\end{equation}

In the following subsections, we derive the lower and upper bounds for the total cost associated with the PIL policy.

\subsubsection{PIL Policy Lower Bound}
For any feasible policy $\pi$, the expired quantity at timestep $t+m-1$ can be expressed as:
\begin{equation}\label{PIL-9}
O_{t+m-1} = \left(\sum_{i=1}^{m+L}x_{t-L,i} - \sum_{j=t-L}^{t+m-1} D_j + \sum_{j=t-L}^{t+m-1} l_j - \sum_{j=t-L}^{t+m-2} O_j\right)^+ \geq \left(\sum_{i=1}^{m+L}x_{t-L,i} - \sum_{j=t-L}^{t+m-1} D_j\right)^+ - \sum_{j=t-L}^{t+m-2} O_j,
\end{equation}
where $\sum_{i=1}^{m+L}x_{t-L,i}=s+\sum _{j=t-L}^{t} d_{j}$ for the OUT policy with an order-up-to level $s+\sum _{j=t-L}^{t} d_{j}$.

Applying the equality in Equation \ref{PIL-9} to the OUT policy gives:
\begin{equation}
\sum_{j=t-L}^{t+m-1} O_j \geq \left(s - \sum_{j=t-L}^{t+m-1} \xi_j - \sum_{j=t+1}^{t+m-1} d_j\right)^+ \geq \left(s - \sum_{j=t-L}^{t+m-1} \xi_j\right)^+ - \sum_{j=t+1}^{t+m-1} d_j.
\end{equation}

This expression can also be derived as:
\begin{equation}\label{PIL-10}
O_{[t,t+m+L-1]}(\mathbf{x}_t) = \sum_{j=t}^{t+m+L-1}O_{j}=\max\left\{\sum_{i=1}^{m+L}x_{t,i}-D_{[t,t+m+L-1]}, O_{[t,t+m+L-2]}(\mathbf{x}_t)\right\} \geq \left(\sum_{i=1}^{m+L}x_{t,i}-\sum_{j=t}^{t+m+L-1}D_{j}\right)^+.
\end{equation}

Taking the expectation of the left- and right-hand sides yields:
\begin{equation}\label{PIL-11}
(m+L)\mathbb{E}[O_t] \geq \mathbb{E}\left[\left(s-\mathcal{D}_{m+L}\right)^+\right]-\sum_{j=t+1}^{t+m-1}d_{j}.
\end{equation}

Similarly,
\begin{equation}\label{PIL-12}
\left(\sum_{i=1}^{m}x_{t+L,i}-D_{t+L}\right)^{+}=\left(\sum_{i=1}^{m+L}x_{t,i}-\sum_{j=t}^{t+L}D_{j}+\sum_{j=t}^{t+L-1}l_{j}-\sum_{j=t}^{t+L-1}O_j\right)^{+} 
\geq \left(\sum_{i=1}^{m}x_{t+L,i}-\sum_{j=t}^{t+L}D_{j}\right)^{+}-\sum_{j=t}^{t+L-1}O_{j}.
\end{equation}

For the OUT policy with an order-up-to level $s+\sum _{j=t}^{t+L} d_{j}$, this implies:
\begin{equation}\label{PIL-13}
\mathbb{E}\left[\left(\sum_{i=1}^{m}x_{t+L,i}-D_{t+L}\right)^{+}\right]\geq \mathbb{E}\left[\left(s-\mathcal{D}_{L+1}\right)^+\right]-L\mathbb{E}[O_t],
\end{equation}
and
\begin{equation}\label{PIL-14}
\left(D_{t+L}-\sum_{i=1}^{m}x_{t+L,i}\right)^{+}=\left(\sum_{j=t}^{t+L}D_{j}-\sum_{j=t}^{t+L-1}l_{j}+\sum_{j=t}^{t+L-1}O_{j}-\sum_{i=1}^{m+L}x_{t,i}\right)^{+} \geq \left(\sum_{j=t}^{t+L}D_{j}-\sum_{i=1}^{m+L}x_{t,i}\right)^{+}-\sum_{j=t}^{t+L-1}l_{j},
\end{equation}

Thus:
\begin{equation}\label{PIL-15}
(L+1)\mathbb{E}\left[\left(D_{t+L}-\sum_{i=1}^{m}x_{t+L,i}\right)^{+}\right]\geq \mathbb{E}\left[\left(\mathcal{D}_{L+1}-s\right)^+\right].
\end{equation}

Combining the above inequalities, we can establish a bound for the total cost as follows:
\begin{equation}\label{Eq: tot cost PIL 1}
\begin{aligned}
C^{\pi_u} &= \mathbb{E}\left[\sum_{t=1}^{T}(h(Y_{t}-D_t)^{+}+p(D_t-Y_{t})^{+}+wO_t)\right] \\
&\geq\sum_{t=1}^{T}\left(h\mathbb{E}[(s-\mathcal{D}_{L+1})^{+}]+\frac{p}{L+1}\mathbb{E}[(\mathcal{D}_{L+1}-s)^{+}]+\frac{w-hL}{m+L}\mathbb{E}[(s-\mathcal{D}_{m+L})^{+}]
-\frac{w-hL}{m+L}\sum_{j=t+1}^{t+m-1}d_{j}\right).
\end{aligned}
\end{equation}

Thus:
\begin{equation}\label{Eq: tot cost PIL 2}
\begin{aligned}
C^{\pi_u} &= \mathbb{E}\left[\sum_{t=1}^{T}(h(Y_{t}-D_t)^{+}+p(D_t-Y_{t})^{+}+wO_t)\right] \\
&\geq T\left(h\mathbb{E}[(s-\mathcal{D}_{L+1})^{+}]+\frac{p}{L+1}\mathbb{E}[(\mathcal{D}_{L+1}-s)^{+}]+\frac{w-hL}{m+L}\mathbb{E}[(s-\mathcal{D}_{m+L})^{+}] -\frac{w-hL}{m+L}\sum_{t=1}^{T}\sum_{j=t+1}^{t+m-1}d_{j}\right).
\end{aligned}
\end{equation}

\subsubsection{PIL Policy Upper Bound}
We conclude this subsection by providing an upper bound on the total cost for the PIL policy. Based on the findings of \cite{Bu2023_4638265}, we note that Lemmas 2 and 4 remain applicable when demand is non-stationary. Specifically, by their Lemma 2, we have $\mathbb{E}\left[\sum_{i=1}^m x_{t+L,i}\right]=U_{t+L}=u+d_{t+L}$, and by the material conservation law, we can express:
\begin{equation}
\begin{aligned}
\mathbb{E}\left[\sum_{i=1}^m x_{t,i}\right] &= d_t+u = \mathbb{E}\left[\sum_{i=1}^{m+L-1} x_{t-L,i}+q_{t-L}-\sum_{j=t-L}^{t-1} \left(D_j-l_j+O_j\right)\right] \\
&= \sum_{i=1}^{m+L-1} x_{t-L,i}+q_{t-L}-\sum_{j=t-L}^{t-1} d_j-\mathbb{E}\left[\sum_{j=t-L}^{t-1} \left(O_j-l_j\right)\right].
\end{aligned}
\end{equation}

This implies:
\begin{equation}
\begin{aligned}
\sum_{i=1}^{m+L-1} x_{t-L,i}+q_{t-L} &=u+ \sum_{j=t-L}^{t} d_j+\mathbb{E}\left[\sum_{j=t-L}^{t-1} \left(O_j-l_j\right)\right].
\end{aligned}
\end{equation}

By Lemma 4 (part b) of \cite{Bu2023_4638265}, the variables $-\sum_{j=t-L}^{t-1} O_j$, $\sum_{j=t-L}^{t-1} l_j$, and $\sum_{j=t-L}^{t-1} (D_j- l_j +O_j)$ are component-wise increasing in $D_j$, $j=t, \ldots,t+L-1$. We thus define the following associated random variables:
\begin{equation}
\begin{aligned}
A_t &= \sum_{j=t-L}^{t-1} \left(-l_j+O_j+\xi_t\right) - \mathbb{E}\left[\sum_{j=t-L}^{t-1} \left(O_j-l_j\right)\right] - \left(d_t+u\right) + d_t + \xi_t, \\
B_t &= \sum_{j=t-L}^{t-1} l_j - \mathbb{E}\left[\sum_{j=t-L}^{t-1} l_j\right], \\
R_t &= \mathbb{E}\left[\sum_{j=t-L}^{t-1} O_j\right] - \sum_{j=t-L}^{t-1} O_j.
\end{aligned}
\end{equation}

Since $\mathbb{E}\left[B_t\right]=0$ and $\mathbb{E}\left[R_t\right]=0$, by Lemma 7 of \cite{vanJaarsveld2024}, we can write:
\begin{equation}
\begin{aligned}
\mathbb{E}\left[\left(D_t-\sum_{i=1}^m x_{t,i}\right)^+\right] &= \mathbb{E}\left[\left(A_t\right)^+\right] \leq \mathbb{E}\left[\left(A_t+B_t\right)^+\right] \leq \mathbb{E}\left[\left(A_t+B_t+R_t\right)^+\right] \\ &\leq \mathbb{E}\left[\left(\sum _{j=t-L}^t \xi _j-u\right)^+\right] = \mathbb{E}\left[\left(\mathcal{D}_{L+1}-u\right)^+\right],
\end{aligned}
\end{equation}
and
\begin{equation}
\begin{aligned}
\mathbb{E}\left[\left(\sum_{i=1}^m x_{t,i}-D_t\right)^+\right] &= \mathbb{E}\left[\left(D_t-\sum_{i=1}^m x_{t,i}\right)^++\sum_{i=1}^m x_{t,i}-D_t\right] =\mathbb{E}\left[\left(D_t-\sum_{i=1}^m x_{t,i}\right)^+\right]+\mathbb{E}\left[\sum_{i=1}^m x_{t,i}\right]-\mathbb{E}\left[D_t\right] \\
&\leq \mathbb{E}\left[\left(\sum_{j=t-L}^t \xi_j-u\right)^+\right]+u =\mathbb{E}\left[\left(\sum_{j=t-L}^t \xi_j-u\right)^++u-\sum_{j=t-L}^t \xi_j\right]\\
&= \mathbb{E}\left[\left(u-\sum_{j=t-L}^t \xi_j\right)^+\right]=\mathbb{E}\left[\left(u-\mathcal{D}_{L+1}\right)^+\right].
\end{aligned}
\end{equation}

Thus, we can bound the expected expired quantity as:
\begin{equation}
\sum_{j=t}^{m+t-1} O_j\leq \left(\sum_{i=1}^m x_{t,i}-D_t\right)^+.
\end{equation}

Taking the expectation of the left- and right-hand sides yields:
\begin{equation}
m\mathbb{E}\left[O_t\right]\leq \mathbb{E}\left[\left(\sum_{i=1}^m x_{t,i}-D_t\right)^+\right]\leq \mathbb{E}\left[\left(u-\sum_{j=t-L}^t \xi_j\right)^+\right]=\mathbb{E}\left[\left(u-\mathcal{D}_{L+1}\right)^+\right].
\end{equation}

Combining the inequalities above, we can finally upper bound the total cost of the PIL policy as:
\begin{equation} \label{eq:ub_pil}
C^{\pi_u} \leq T\left((h+\frac{w}{m}) \mathbb{E}\left[\left(u-\mathcal{D}_{L+1}\right)^+\right]+ b \mathbb{E}\left[\left(\mathcal{D}_{L+1}-u\right)^+\right]\right).
\end{equation}

\subsubsection{PIL Implementation}
The proposed PIL policy requires estimating the expected projected inventory level $L$ timesteps into the \textit{future} to calculate the order quantity. This approach involves computing the expected lost sales and expired stock over $L$ timesteps. For this purpose, we adopt a methodology aligned with \cite{Zhang2016} and \cite{Bu2023_4400}. Similar to the OUT policy, the safety stock employed by the PIL policy is determined using Monte Carlo simulations tailored to its specific requirements (see Section \ref{sec:bs} for details).

To estimate the expected inventory level at timestep $t+L$, given the current inventory vector (i.e., $\mathbb{E}\left[\sum_ {i=1}^{m} x_{t+L,i} \mid \mathbf{x_t}\right]$), we perform 2000 simulated episodes over the entire episode horizon $T$, generating 2000 demand sample paths for each timestep $t$. Specifically, following the order of events defined in our supply chain environment (Section \ref{sec:section3}), at each timestep $t$, we: \textit{i)} set the actual order quantity $n_t$ to zero, as it does not affect the computation of the expected inventory level at timestep $t+L$ but impacts timestep $t+L+1$; \textit{ii)} update the total inventory level based on received inventory in transit, as represented in Equation \ref{Eq:Dynamic}; and \textit{iii)} track any potential lost sales and expired stock. Upon completing the simulations for each timestep within the interval from $t$ to $t+L$, we compute the average amounts of expired stock and lost sales over this interval. These averages provide the expected inventory level at timestep $t+L$ necessary for determining the order quantity $n_t$, as expressed in Equations \ref{PIL-6} and \ref{PIL-8}.

\subsection{PPO Algorithm} \label{sec:ppo}
PPO is an actor-critic algorithm designed to enhance the efficiency and stability of policy updates \citep{schulman2017proximal}. It modifies the objective function to prevent large updates to the policy by introducing a \textit{clipped objective function}:
\begin{equation}
L^{CLIP}(\theta) = \hat{\mathbb{E}}_t \left[\min\left(r_t(\theta) \hat{A}_t, \text{clip}\left(r_t(\theta), 1 - \epsilon, 1 + \epsilon\right) \hat{A}_t\right)\right],
\end{equation}
where $r_t(\theta) = \frac{\pi_{\theta_{new}}(a_t|s_t)}{\pi_{\theta_{old}}(a_t|s_t)}$ represents the probability ratio between the new policy $\pi_{\theta_{new}}$ and the old policy $\pi_{\theta_{old}}$ for taking action $a_t$ in state $s_t$. The hyperparameter $\epsilon$ controls the clipping range to ensure that updates remain constrained. Specifically, the function $\text{clip}(x, 1 - \epsilon, 1 + \epsilon)$ constrains $x$ within the interval $[1 - \epsilon, 1 + \epsilon]$. This mechanism limits the magnitude of policy updates, enhancing stability and reducing the risk of divergence during training.

The advantage estimate $\hat{A}_t$ is computed using \textit{generalized advantage estimation} (GAE), which balances bias and variance in the estimation process:
\begin{equation}
\hat{A}_t = \delta_t + (\gamma \lambda) \delta_{t+1} + \cdots + (\gamma \lambda)^{T-t+1} \delta_{T-1},
\end{equation}
where $\delta_t = r_t + \gamma V_t(s_{t+1}) - V_t(s_t)$ represents the temporal difference error at timestep $t$, $\gamma$ is the discount factor, $\lambda$ is the GAE parameter, and $T$ is the episode horizon. By incorporating cumulative future advantages, GAE enables more nuanced policy updates, optimizing decisions in a stable and computationally efficient manner. The combination of the clipped objective function and GAE forms the basis of PPO, making it well-suited for addressing complex inventory problems. For a more detailed discussion of the PPO algorithm and its theoretical foundations, readers are encouraged to refer to \cite{schulman2017proximal}.

\subsubsection{PPO Implementation}
To implement the PPO algorithm and learn a well-performing policy, $\pi_{PPO}$, we formalize the supply chain environment as a Markov decision process (MDP) in terms of features (observations), actions, and rewards.

As is standard in inventory management representations, the state is defined as a feature vector (the \textit{observation}) serving as input to the neural network \cite[see][]{Boute2022}. Conventionally, features would consist of the inventory vector $\mathbf{x}_t$ and the timestep $t$, which together form a \textit{sufficient statistic} for the state. However, relying solely on this information places the entire burden on the neural network to learn the product life-cycle phases and interpret the inventory vector accordingly. Preliminary experiments revealed that, in our case study, this approach results in poor performance, as the neural network struggles to manage these complexities effectively.

To overcome this constraint, we propose enriching the feature vector by including a \textit{forecast of upcoming demands}, similar to those used in the OUT and PIL policies. Additionally, rather than directly using the inventory vector $\mathbf{x}_t$ at timestep $t$, we adopt projection-based ideas inspired by the literature \citep{Broekmeulen2009,vanJaarsveld2024} to estimate the future inventory vector at timestep $t+L$, just before the order placed at timestep $t$ arrives.

As a result, the \textit{feature vector} at timestep $t$ is defined as:
\begin{equation} \label{eq:state1}
\left( \mathbb{E}[\mathbf{x}_{t+L}], d_{[t:t+L]}, t \right),
\end{equation}
where $\mathbb{E}[\mathbf{x}_{t+L}] = \left( \mathbb{E}[x_{t+L,1}], \mathbb{E}[x_{t+L,2}], \ldots, \mathbb{E}[x_{t+L,m}] \right)$ represents the \textit{expected inventory levels} at timestep $t+L$ for each age category $i \in \{1, 2, \ldots, m\}$. These expected values are computed under the assumption that upcoming demands equal their forecasts, enabling a deterministic estimation. Additionally, $d_{[t:t+L]} = (d_{t,t+1}, d_{t,t+2}, \ldots, d_{t,t+L})$ represents the demand forecasts from timestep $t$ to $t+L$, and $t$ is the current timestep. Preliminary experiments demonstrated that this new feature representation outperforms standard representations relying solely on current inventory levels \citep[see, e.g.,][]{Oroojlooyjadid2022,Gijsbrechts2022,DeMoor2022,Kaynov2024,Stranieri2024_IJPE,Stranieri2024}.

The \textit{action} $a_t$ at timestep $t$ determines the order quantity:
\begin{equation}
q_t^{\pi_{PPO}}=a_tQ=n_t^{\pi_{PPO}}Q,
\end{equation}
where $a_t$ represents the number of batches ordered, chosen from the discrete set $\{0, 1, 2, \ldots, 6\}$. Here, an upper limit of 6 reflects the maximum number of batches that can be ordered in the BMS case study.

Finally, the \textit{reward} $r_t \in \mathbb{R}$ at timestep $t$ is calculated based on the following cost components:
\begin{equation}
r_t = - \left(K(q_t) + h(Y_{t} - D_t)^{+} + b(D_t - Y_{t})^{+} + wO_t\right),
\end{equation}
where $K(q_t)$ represents the batch ordering cost, $h$ is the holding cost, $b$ is the lost sales cost, and $w$ is the expiration cost, as detailed in Section \ref{sec:section3}. Note that we define the reward with a negative sign since the objective is to minimize the total cost.

\section{Numerical Experiments} \label{sec:section5}
In this section, we present \textit{numerical experiments} designed to evaluate the performance of the implemented policies under various demand scenarios based on synthetic data provided by BMS. All experiments were conducted on a machine equipped with an 11th Gen Intel(R) Core(TM) i7-11800H CPU at 2.30 GHz and 32 GB of RAM. The code was developed using Python 3.10, leveraging the OpenAI Gym library \citep{brockman2016openai} to define the supply chain environment and the Stable Baselines 3 library \citep{stable-baselines} to implement the PPO algorithm. The code is publicly available as an open-source library on GitHub \footnote{\url{https://github.com/frenkowski/SCIMAI-Gym}.}.

In detail, we consider two demand scenarios for the numerical experiments, both modeled as non-stationary to reflect the \textit{complete lifecycle} of a specific perishable pharmaceutical drug, as discussed with BMS. The lifecycle begins with initial growth as the product gains market acceptance, reaches a peak during its maturity phase, and then declines as alternatives emerge or patents expire. This demand pattern closely mirrors real-world conditions, providing a realistic basis for evaluating the proposed policies in pharmaceutical supply chains.

The \textit{first scenario} is adapted from company-provided data and covers a period of 5 years ($T=60$ monthly timesteps), incorporating two levels of demand noise. In the \textit{worst-case setting}, demand noise is modeled as $\xi_t \sim \mathcal{N}(0, \bar{d} \times 15\%)$, where $\bar{d} = \max_t{d_t}$, representing high fluctuations during the peak phase. In contrast, the \textit{balanced setting} models noise as $\xi_t \sim \mathcal{N}(0, d_t \times 15\%)$, reflecting more moderate fluctuations. These settings were chosen to evaluate the effectiveness of the implemented policies under varying levels of uncertainty. Specifically, the worst-case setting assesses policy performance under high uncertainty, while the balanced setting represents more stable market conditions. In both cases, the forecast error $\xi_t$ follows i.i.d. normal distributions with a mean of zero and a standard deviation $\sigma$. This assumption is essential for deriving the lower and upper bounds necessary to optimize the OUT and PIL policies, as discussed in Sections \ref{sec:bs} and \ref{sec:pil}.

The \textit{second scenario} is based on real-world data provided by BMS, modeling the complete lifecycle of a specific perishable pharmaceutical drug over a 20-year period ($T=240$ monthly timesteps), as illustrated in Figure \ref{fig:demand_2}. This scenario employs the same balanced noise setting described in the first scenario and serves as a reference for more practical and comprehensive analysis.

For each experiment, we compare the performance of the policies by analyzing the \textit{average total cost} over 2000 simulated episodes. By assessing performance across varying levels of demand noise and over both short-term and long-term horizons, this study provides valuable insights into the cost-effectiveness and robustness of the implemented policies in addressing typical pharmaceutical inventory challenges.

\begin{figure}[ht!]
    \centering
     \includegraphics[width=0.65\linewidth]{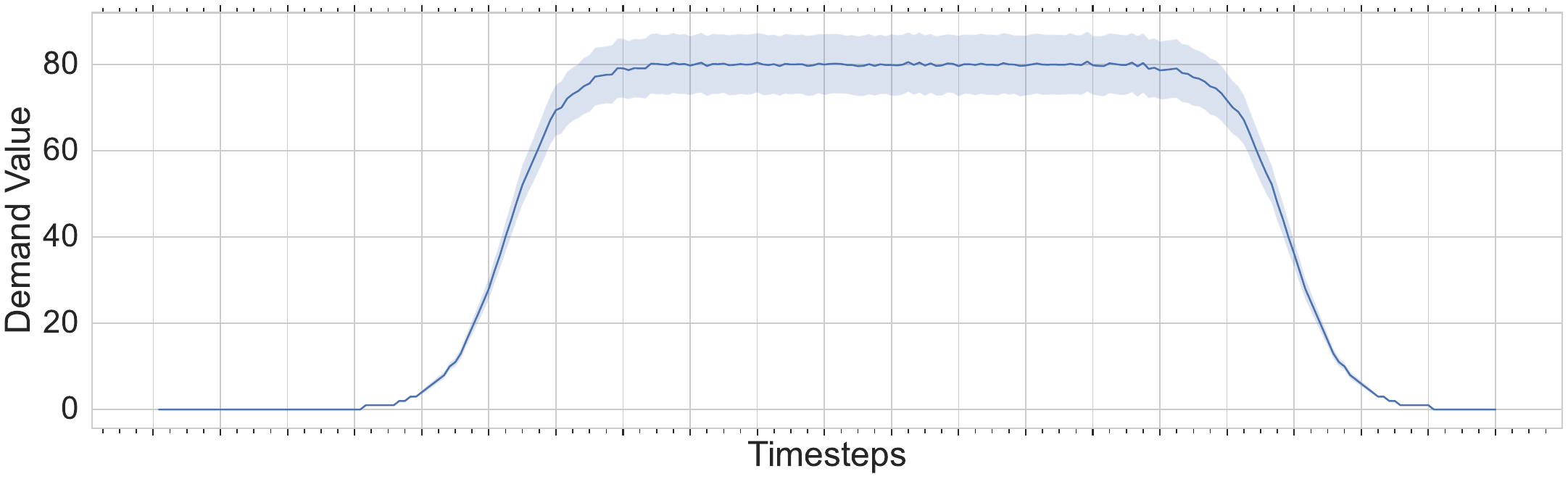}
        \caption{A 95\% confidence interval based on 2000 simulated episodes for the demand in the second scenario derived from real-world data over an episode horizon of $T = 240$ timesteps.}
        \label{fig:demand_2} 
\end{figure}

\subsection{First Scenario}
In the first scenario, we evaluate the three inventory policies introduced in Section \ref{sec:section4}--OUT, PIL, and PPO--under the worst-case and balanced settings. The \textit{experimental plan} involves varying several parameters of the supply chain environment. Specifically, the expiration cost ($w$) is set to 2 and 4, the product lifetime ($m$) ranges from 2 to 4, and the lost sales cost ($b$) takes on values of 10, 50, 100, and 1000. Other parameters are kept constant to simplify the analysis. The holding cost ($h$) is fixed at 1, the lead time ($L$) is set to 2, and both the batch ordering cost ($K(q_t)$) and yield rate ($\hat{z}$) are set to 0. This setup results in 24 unique experiments for each of the two settings (worst-case and balanced). These experiments systematically evaluate the policies under varying cost structures and product characteristics, providing valuable insights into their performance across diverse market conditions.

\subsubsection{Worst-Case Setting}
The bounds derived for the optimal parameters in the worst-case setting simplify the optimization process by constraining the \textit{search interval} for the OUT and PIL policies. As shown in Table \ref{tab:results_bounds} of \ref{appendix:4}, the optimal values for both the OUT and PIL policies consistently fall within this interval. These bounds are particularly advantageous for the PIL policy, which requires more computational resources than OUT due to the need to calculate expected inventory levels. Although these bounds were derived based on the total cost, they also provide valuable estimates for the optimal parameters of both policies. While formal proof that the optimal values strictly lie within this interval is lacking, the numerical results strongly support this conjecture.

In this first scenario, since there is no upper limit on the number of batches that can be ordered, the upper bound constraining the \textit{action space} of the PPO algorithm was determined by selecting the maximum of the optimal values between the OUT and PIL policies, which minimize Equation \ref{eq:ub_out} and Equation \ref{eq:ub_pil}, respectively. This choice is motivated by the observation that optimal values tend to be lower for non-perishable products. Consequently, the non-perishable bounds derived from the OUT and PIL policies were effectively used to constrain the PPO action space.

Turning to the results, as depicted in Figures \ref{fig:results_bounds} and \ref{fig:results_without_std}, when the product lifetime $m$ is set to 2, PIL and PPO exhibit similar performance, both slightly outperforming OUT when the lost sales cost $b$ is low. However, as $b$ increases, PPO significantly outperforms OUT while maintaining competitiveness with PIL, with the performance gap widening as $b$ reaches its highest value. As the product lifetime increases to $m=3$, PIL remains the most effective policy as lost sales costs rise, with PPO performing on par with it while both consistently outperform OUT. Finally, when the product lifetime extends to $m=4$, the performance gap among the three policies tends to narrow. PIL continues to excel, particularly at the highest value of $b$. In contrast, PPO's performance slightly declines, even compared to OUT in some cases when the expiration cost $w$ is set to 2. However, OUT demonstrates higher variability in terms of standard deviation, reflecting its sensitivity to cost fluctuations, especially when $b=1000$.

\begin{figure}[ht!]
    \centering
     \includegraphics[width=0.70\linewidth]{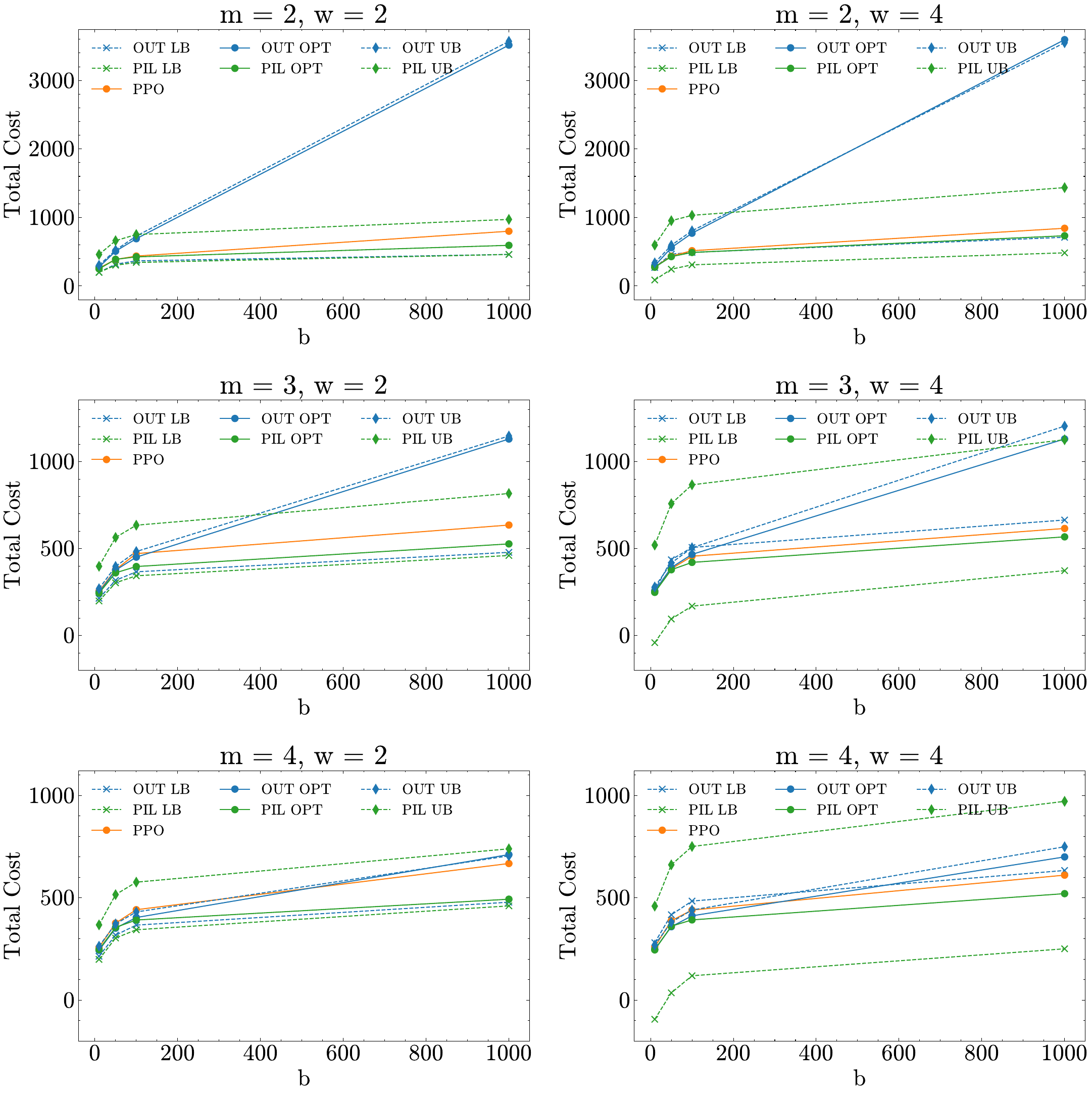}
        \caption{Average total cost for the PPO algorithm, with lower (LB), upper (UB), and optimal (OPT) values for the OUT and PIL policies. Demand noise is modeled as $\xi_t \sim \mathcal{N}(0, \bar{d} \times 15\%)$, where $\bar{d} = \max_t{d_t}$. Each row corresponds to a different value of $m = \{2, 3, 4\}$, and each column to a different value of $w = \{2, 4\}$. Each subplot shows the OUT, PIL, and PPO costs for $b = \{10, 50, 100, 1000\}$.}
        \label{fig:results_bounds} 
\end{figure} 

\begin{figure}[ht!]
    \centering
     \includegraphics[width=0.65\linewidth]{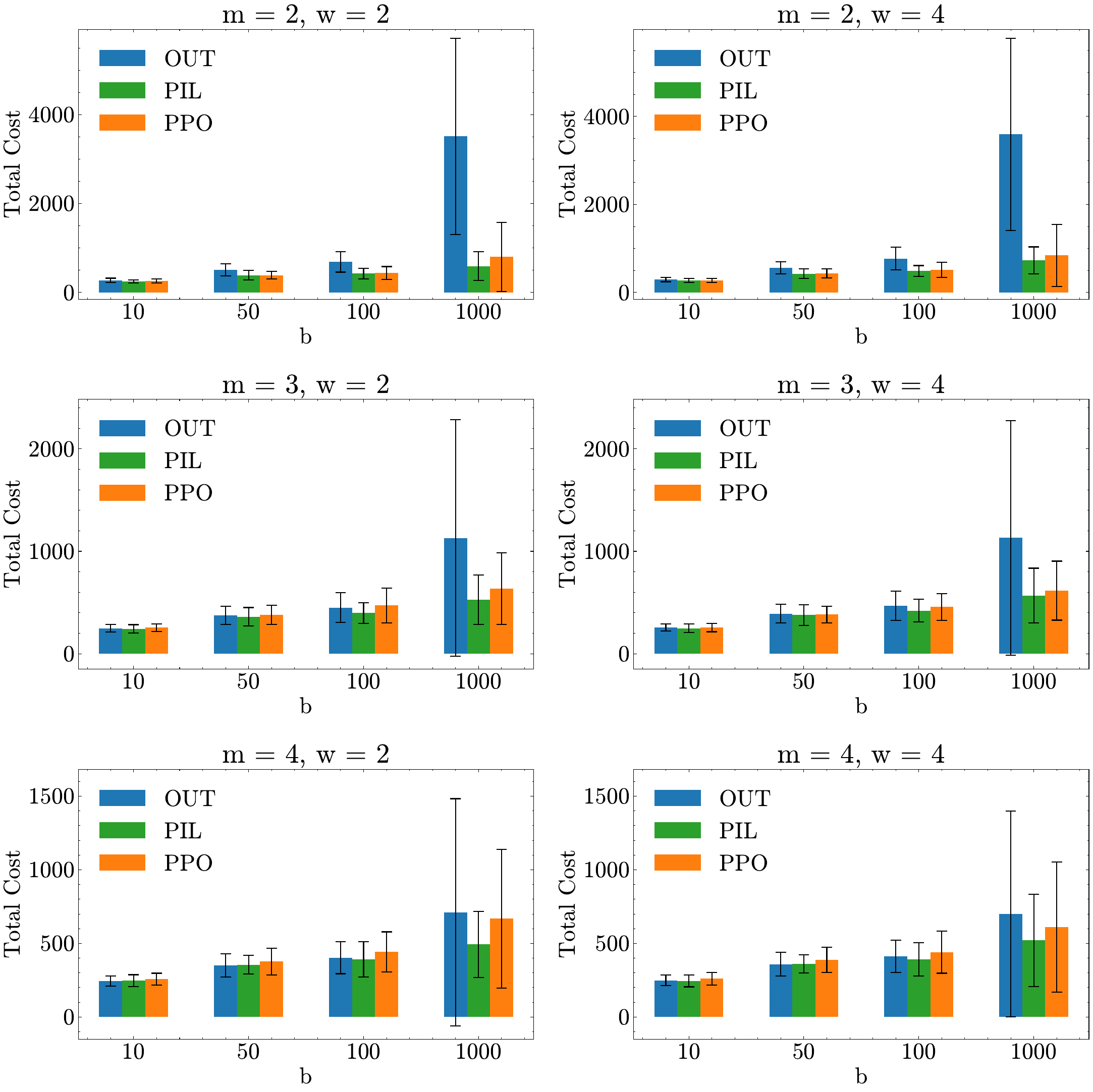}
        \caption{Bar plots representing the average total cost over 2000 simulated episodes, with demand noise modeled as $\xi_t \sim \mathcal{N}(0, \bar{d} \times 15\%)$, where $\bar{d} = \max_t{d_t}$. Each row corresponds to a different value of $m = \{2, 3, 4\}$, and each column to a different value of $w = \{2, 4\}$. In each subplot, the bars show the OUT, PIL, and PPO costs for $b = \{10, 50, 100, 1000\}$.}
        \label{fig:results_without_std} 
\end{figure}

Note that the OUT policy exhibits significant variability, primarily driven by fluctuations in lost sales costs. Specifically, as lost sales costs increase, even minor changes between episodes can substantially impact the average total cost. In contrast, expiration and holding costs remain relatively stable and smaller in magnitude. The \textit{dominance} of lost sales costs means that differences or outliers disproportionately affect the standard deviation. This variability is likely due to demand uncertainty, which overshadows the stability of other cost components, resulting in significant variability in the average total cost, particularly when lost sales costs reach their highest values.

\subsubsection{Balanced Setting}
To analyze the performance of the three implemented policies further, we evaluate them under more moderate fluctuations. Figure \ref{fig:results_with_std} presents the results for this balanced setting. When the product lifetime $m$ is set to 2, the OUT policy significantly underperforms under high lost sales costs and exhibits a high standard deviation, consistent with its behavior in the worst-case setting. In contrast, PPO consistently achieves lower total costs compared to PIL, particularly as lost sales costs increase.

As the product lifetime increases to $m=3$, PPO continues to achieve the lowest total cost. The OUT and PIL policies perform similarly, with OUT occasionally outperforming PIL, particularly when $b=1000$, although these differences remain minimal.

Finally, when the product lifetime is extended to $m=4$, and the expiration cost $w$ is set to 2, all three policies yield nearly identical total costs. However, when the expiration cost increases to $w=4$, and the lost sales cost $b$ is set to 1000, PPO performs slightly worse than the other policies, while PIL and OUT maintain consistent performance levels.

\begin{figure}[ht!]
    \centering
     \includegraphics[width=0.65\linewidth]{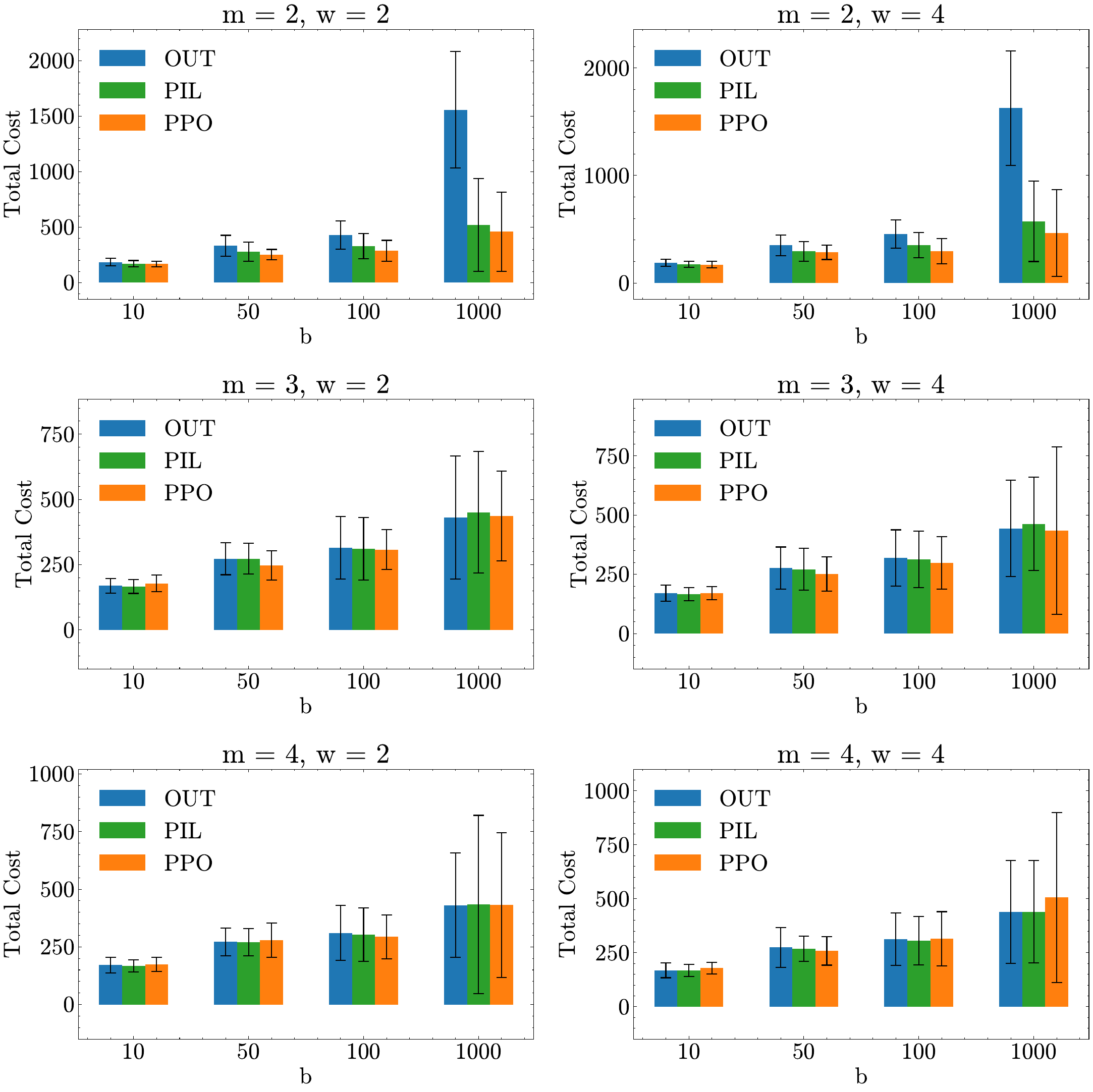}
        \caption{Bar plots representing the average total cost over 2000 simulated episodes, with demand noise modeled as $\xi_t \sim \mathcal{N}(0, d_t \times 15\%)$. Each row corresponds to a different value of $m = \{2, 3, 4\}$, and each column to a different value of $w = \{2, 4\}$. In each subplot, the bars show the OUT, PIL, and PPO costs for $b = \{10, 50, 100, 1000\}$.}
        \label{fig:results_with_std} 
\end{figure} 

Interestingly, the OUT policy shows a slight improvement in this balanced setting, emerging as the most efficient policy in specific experiments, even when lost sales costs are high. This improvement is particularly evident in experiments with longer product lifetimes ($m = 3$ or $m = 4$). However, it continues to underperform in experiments with the shortest product lifetime ($m = 2$) as lost sales costs reach $b=1000$. The PPO algorithm performs efficiently in nearly all experiments, although it exhibits a slight decline in performance at the highest expiration and lost sales costs with the longest product lifetimes. This behavior is consistent with observations in the worst-case setting. The slight decline, while expected, may also be partially attributed to challenges in hyperparameter tuning.

The main \textit{insight} from both the worst-case and balanced settings is that no single policy consistently dominates across all experiments. Instead, each policy has experiments where it outperforms the others. Furthermore, the asymptotic optimality of the PIL policy, observed for non-perishable products, does not necessarily extend to perishable products.

\subsection{Second Scenario}
In the second scenario, we evaluate the performance of the OUT, PIL, and PPO policies against the BMS baseline derived from \textit{human expertise}. This human-driven policy is based on decisions made by a team of expert human planners at BMS with extensive knowledge of managing perishable pharmaceutical drugs, as discussed in Section \ref{sec:bmspol}. These decisions rely on the same information provided to the PPO algorithm through Equation \ref{eq:state1}, ensuring a fair comparison.

A key distinction between DRL-based approaches, such as PPO, and the human-driven policy lies in adaptability. While human planners' decisions are built on expertise, they remain \textit{static and are provided a priori}. Consequently, the human-driven policy cannot adjust dynamically in response to the actual state of the supply chain environment or react to unplanned demand fluctuations. In fact, these static decisions remain fixed throughout the simulated episodes, limiting their ability to manage the variability and uncertainty typically encountered in real-world conditions. In contrast, PPO offers a more flexible and responsive solution, \textit{dynamically adapting} at each timestep based on the current state of the environment. Despite its simplicity, the human-driven policy serves as a valuable benchmark, highlighting its limitations and providing a reference point for evaluating advanced, data-driven policies. This comparison underscores the potential for significant improvements in efficiency and responsiveness, which are particularly relevant to pharmaceutical companies.

Table \ref{tab:bms_5} summarizes the performance of the four implemented policies (OUT, PIL, PPO, and the BMS baseline) under various yield rates ($\hat{z} = 0\%, 5\%, 10\%$) and lost sales costs ($b = 100, 1000, 10000$) across nine different experiments. All experiments assume a lead time of $L = 12$ timesteps, a product lifetime of $m = 12$ timesteps, and an expiration cost of $\hat{w} = 3.0$. It is important to note that the human-driven policy was explicitly designed for a specific experiment (with $\hat{z} = 10\%$ and $b = 100$) and then applied without modification to the other experiments.

\begin{table}[ht!]
\centering
\scriptsize
\caption{Average total cost for OUT, PIL, PPO, and the human-driven policy over 2000 simulated episodes in the second scenario derived from real-world data, for values of $\hat{z} = \{0\%, 5\%, 10\%\}$ and $b = \{100, 1000, 10000\}$. Lower cost values indicate better performance.}
\label{tab:bms_5}
\renewcommand{\arraystretch}{0.5} 
\begin{tabular}{c l l l l l}
\toprule
\multicolumn{1}{c}{$\hat{z}$}    & \multicolumn{1}{c}{$b$}     & \multicolumn{1}{c}{OUT}                     & \multicolumn{1}{c}{PIL}                    & \multicolumn{1}{c}{PPO}                    & \multicolumn{1}{c}{Human}                 \\ \midrule
\multirow{3}{*}{\centering 0\%}      
         & 100     & 19392 $\pm$  3438     & 21457 $\pm$  3594     & 17594 $\pm$  2981     & 120538 $\pm$   758    \\ 
         & 1000    & 28182 $\pm$ 10912     & 31824 $\pm$ 10403     & 25125 $\pm$  9527     & 120538 $\pm$   758    \\ 
         & 10000   & 35580 $\pm$ 23507     & 40011 $\pm$ 32987     & 33829 $\pm$ 22952     & 120538 $\pm$   758    \\ \midrule
\multirow{3}{*}{\centering 5\%}      
         & 100     & 20853 $\pm$  3228     & 23752 $\pm$  3909     & 18861 $\pm$  2796     & 100968 $\pm$  1102    \\ 
         & 1000    & 29796 $\pm$  5478     & 34195 $\pm$ 11336     & 28069 $\pm$  9410     & 100968 $\pm$  1102    \\ 
         & 10000   & 35869 $\pm$ 21999     & 41497 $\pm$ 22702     & 36070 $\pm$ 25478     & 100968 $\pm$  1102    \\ \midrule
\multirow{3}{*}{\centering 10\%}     
         & 100     & 22588 $\pm$  3349     & 26349 $\pm$  3852     & 18608 $\pm$  2838     & \phantom{0}77477 $\pm$  1170     \\ 
         & 1000    & 30847 $\pm$  9606     & 36755 $\pm$ 10183     & 29319 $\pm$  9450     & \phantom{0}80613 $\pm$  1890     \\ 
         & 10000   & 37415 $\pm$ 10711     & 44803 $\pm$ 27323     & 37061 $\pm$ 33831     & 113766 $\pm$ 20326    \\ 
\bottomrule
\end{tabular}
\end{table}

When the yield rate is 0\%, all implemented policies outperform the human-driven policy. Despite the lost sales cost, PPO consistently achieves the lowest costs and standard deviations across all experiments. Interestingly, the OUT policy incurs lower costs than PIL, with the performance gap widening as $b$ increases.

At a 5\% yield rate, the costs of the human-driven policy decrease due to the higher production yield, which reduces accumulated stock compared to the $\hat{z} = 0\%$ case. Here, PPO continues to outperform PIL across all experiments, while OUT achieves the lowest costs in the most challenging experiment, where $b = 10000$.

When the yield rate extends to 10\%, the performance gap between the human-driven policy and the others narrows. PPO remains the most cost-effective policy across all experiments. However, OUT performs comparably to PPO when $b = 10000$ and exhibits a significantly lower standard deviation than the other policies.

As previously discussed, the high total costs associated with the BMS baseline stem from its static nature, which prevents dynamic adaptation to changes in the supply chain environment. In fact, human planners rely on predetermined actions at each timestep, which remain unchanged throughout the simulated episodes. An analysis of the experiment explicitly designed for the human-driven policy reveals a \textit{tendency to accumulate excess on-hand inventory}. On one hand, this approach almost completely avoids lost sales, particularly in experiments with reduced production yields. On the other hand, it results in significantly higher holding costs, explaining the observed cost gap. While the other policies achieve lower average total costs, they tend to exhibit higher standard deviations, primarily due to experiencing lost sales during the simulated episodes. In contrast, the human-driven policy nearly satisfies all demands.

According to BMS, the average total cost may not be the sole \textit{evaluation criterion} for human planners. Instead, they aim to balance the risk of lost sales with product availability, trying to minimize overstocking and associated product expiration while adhering to ethical and legal constraints regarding potential stockouts.

Contrary to the results reported in the first scenario, PIL does not consistently outperform OUT and PPO across all experiments, with the performance gap widening as production yield increases. As demonstrated in \cite{Bu2023_4638265}, PIL achieves \textit{optimality} when the product lifetime is restricted to a single period. In other experiments, PIL tends to underperform compared to OUT, as observed in this second scenario, where both the lead time and product lifetime are set to 12. Indeed, PIL does not always converge to an optimal solution, even as lost sales costs increase.

To investigate these findings further, we conduct a more detailed analysis of the real-world case study using varying key parameters identified in the first scenario. Specifically, we vary product lifetimes from the set $m = \{3, 6, 12, 24\}$ and lost sales costs from the set $b = \{10, 50, 100, 1000, 10000\}$ for each considered expiration cost $w = \{3, 6\}$, and lead times $L = \{3, 6, 12\}$, while maintaining the production yield $\hat{z}$ constant at zero. The results of these experiments provide deeper insights into the strengths and limitations of the implemented policies when applied across a more diversified and extended set of experimental conditions.
    
\begin{table}[ht!]
    \centering
    \scriptsize
    \caption{Percentage gap in performance between OUT vs. PPO and PIL vs. PPO for various values of $w = \{3, 6\}$, $m = \{3, 6, 12, 24\}$, and $b = \{10, 50, 100, 1000, 10000\}$, at lead times $L = \{3, 6, 12\}$. A positive percentage value indicates better performance by PPO.}
    \label{tab:comparison}
    \renewcommand{\arraystretch}{0.5} 
    \begin{tabular}{c|c|c|c c c c c|c c c c c}
        \toprule
        \multicolumn{3}{c|}{} & \multicolumn{5}{c|}{OUT vs. PPO} & \multicolumn{5}{c}{PIL vs. PPO} \\ \midrule
        $L$ & $w$ & \multicolumn{1}{|c|}{$m$} & \multicolumn{5}{c|}{$b$} & \multicolumn{5}{c}{$b$} \\ 
            &     &                           & 10 & 50 & 100 & 1000 & 10000 & 10 & 50 & 100 & 1000 & 10000 \\ \midrule
        \multirow{8}{*}{3} 
            & \multirow{4}{*}{3} & 3   & 12 & 28 & 30 & 57 & 315 & 14 & 45 & 52 & 52 & 32 \\ 
            &                    & 6   & 10 & 25 & 26 & 29 & 159 & 11 & 35 & 36 & 30 & 7  \\ 
            &                    & 12  & 7  & 24 & 23 & 20 & -8  & 8  & 29 & 26 & 26 & -3 \\ 
            &                    & 24  & 7  & 23 & 23 & 8  & 8   & 7  & 24 & 25 & 10 & 10 \\ \cmidrule{2-13}
            & \multirow{4}{*}{6} & 3   & 15 & 33 & 42 & 67 & 279 & 18 & 58 & 74 & 76 & 43 \\ 
            &                    & 6   & 13 & 33 & 32 & 33 & 101 & 14 & 43 & 45 & 37 & 30 \\ 
            &                    & 12  & 8  & 31 & 29 & 29 & 4   & 10 & 36 & 33 & 35 & 8  \\ 
            &                    & 24  & 7  & 27 & 24 & 18 & -9  & 7  & 28 & 25 & 20 & -7 \\ \midrule
        \multirow{8}{*}{6} 
            & \multirow{4}{*}{3} & 3   & 9  & 17 & 20 & 55 & 204 & 14 & 45 & 51 & 47 & 24 \\ 
            &                    & 6   & 5  & 17 & 23 & 53 & 46  & 6  & 22 & 30 & 39 & 14 \\ 
            &                    & 12  & 13 & 29 & 28 & 60 & 571 & 21 & 60 & 72 & 72 & 64 \\ 
            &                    & 24  & 5  & 21 & 26 & 37 & -13 & 6  & 27 & 36 & 27 & -30 \\ \cmidrule{2-13}
            & \multirow{4}{*}{6} & 3   & 9  & 13 & 15 & 30 & 44  & 9  & 28 & 31 & 29 & 20 \\ 
            &                    & 6   & 7  & 21 & 20 & 9  & -1  & 7  & 23 & 23 & 11 & 3  \\ 
            &                    & 12  & 11 & 22 & 24 & 47 & 53  & 12 & 41 & 47 & 52 & 39 \\ 
            &                    & 24  & 6  & 21 & 16 & 17 & 0   & 6  & 23 & 17 & 22 & 2  \\ \midrule
        \multirow{8}{*}{12} 
            & \multirow{4}{*}{3} & 3   & 8 & 10 & 16 & 67 & 566 & 7 & 34 & 46 & 51 & 28 \\ 
            &                    & 6   & 4 & 9  & 7  & 53 & 113 & -1 & 22 & 24 & 25 & 2  \\ 
            &                    & 12  & 6 & 10 & 10 & 12 & 5   & 1  & 19 & 22 & 27 & 18 \\ 
            &                    & 24  & 3 & 10 & 14 & 23 & 10  & -4 & 10 & 15 & 14 & -3 \\ \cmidrule{2-13}
            & \multirow{4}{*}{6} & 3   & 11 & 14 & 20 & 89 & 499 & 15 & 46 & 62 & 69 & 44 \\ 
            &                    & 6   & 8  & 10 & 11 & 29 & 154 & 4  & 27 & 33 & 40 & 19 \\ 
            &                    & 12  & 0  & 6  & 10 & -1 & -2  & -4 & 15 & 24 & 16 & 10 \\ 
            &                    & 24  & 3  & 7  & 17 & 8  & 21  & -4 & 8  & 18 & -1 & 8  \\ 
        \bottomrule
    \end{tabular}
\end{table}

As shown in Table \ref{tab:comparison}, PPO significantly outperforms OUT for shorter product lifetimes (i.e., $m=3$ or $m=6$), particularly when lost sales costs reach their highest value ($b=10000$). In these experiments, PPO achieves a performance gap of approximately 500\% when the lead time $L$ is set to 12. As product lifetimes increase, OUT's performance improves accordingly, independent of expiration costs, and it slightly surpasses PPO in some experiments where $b=1000$ or $b=10000$.

PIL performs well when lost sales costs are at their lowest value ($b=10$) and product lifetimes are longer (i.e., $m \geq 6$), especially when lead times are at their maximum value ($L = 12$). PIL also occasionally outperforms PPO when $m=12$ or $m=24$ with $b \geq 1000$, although no clear trend emerges. However, as lost sales costs increase to a range between 50 and 100, PIL's performance tends to decline, consistently favoring PPO.

A critical factor influencing PIL's performance is the term $\mathbb{E}\left[\sum_{j=t}^{t+L-1} O_j - l_j\right]$, as defined in Equation \ref{PIL-6}. This term represents the difference between the expected expired quantity and lost sales over the lead time. For non-perishable products, incorporating only the expected expired quantity when computing $q_t$ allows PIL to outperform OUT \citep{Broekmeulen2009}. However, when PIL subtracts the expected lost sales term, this \textit{non-zero mean term} can fluctuate between positive and negative values across timesteps, particularly under non-stationary demand. Such variability makes it challenging for PIL to determine optimal order quantities, leading to higher total costs than OUT, especially when lost sales costs are low (i.e., $b$ equals 10, 50, or 100). In contrast, PIL outperforms OUT by more effectively accounting for expired and lost sales terms for short product lifetimes and high lost sales costs, mainly when $m=3$ and $b=10000$. In experiments with larger lead times and longer product lifetimes, both PIL and OUT tend to increase order quantities to reduce lost sales. However, PIL incorporates projected inventory levels by more accurately calculating the expected expired quantity and lost sales, often resulting in better decisions and lower total costs for both values of $w$. PPO demonstrates an even greater ability to adapt to these trade-offs, frequently outperforming both PIL and OUT. Nevertheless, as lead times and product lifetimes extend, the \textit{dimensionality of PPO's state vector} increases accordingly, making optimization more challenging and reducing its performance gap over the other policies.

\section{Conclusions} \label{sec:section6}
In this study, we evaluated the performance of different inventory control policies for a specific perishable pharmaceutical drug under non-stationary demand, considering factors such as random yields, perishability, positive lead times, and lost sales. This analysis was inspired by a real-world case study provided by BMS. By comparing the OUT policy, the PIL policy, and the PPO algorithm with a human-driven policy, we collected \textit{managerial insights} into their cost-effectiveness and robustness across a diverse set of experimental conditions. The theoretical bounds derived for the OUT and PIL policies offered valuable direction into their optimal parameter values, while the MDP formulation of the supply chain environment was essential for effectively implementing the PPO algorithm.

Extensive numerical experiments revealed that, while simple to implement, the OUT policy exhibits limitations in experiments with short product lifetimes and high lost sales costs. However, its performance improves as these values increase, as observed in the second scenario based on real-world data. The PIL policy offers a robust and cost-effective solution across various experiments, demonstrating consistent performance without divergence under any of the experiments considered. Notably, it proves particularly efficient in the first scenario, adapted from real-world conditions with reduced product lifetime and lead time values. Finally, as observed in the second scenario, the PPO algorithm demonstrated superior performance in more complex and variable experiments, particularly those involving higher lost sales costs and the inclusion of production yields. However, challenges related to hyperparameter tuning, the computational resources required for training, and the time needed for convergence can be significant, especially as the state vector grows with longer product lifetimes and lead times.

Another salient insight from this study is that none of the three policy types consistently outperforms the others across all experimental conditions. While all three implemented policies achieved lower average costs compared to the human-driven policy, they exhibited higher standard deviations. In contrast, the human-driven policy tends to maintain large inventory levels to avoid stockouts, resulting in higher holding costs but significantly lower cost variability. Therefore, decision-makers in pharmaceutical companies should evaluate policy choices based on specific operational contexts and objectives. Focusing solely on average total cost may not be sufficient, as ethical and legal constraints associated with potential stockouts must also be considered.


\bibliographystyle{apalike}
\bibliography{mybibliography}

\appendix

\section{Proof of Lemma 1} \label{appendix:1}
\begin{proof}
We know that
$q_t = 0$ for $t= T - L+1,\ldots,T$,
and 
$\sum_{i=1}^{m-1}x_{T+1,i}=(Y_T-D_T)^+-(X_{T,1}-D_T)^+$.
Additionally, the following relations hold:
$q_{t-L} = (Y_{t+L-1,1} - D_{t+L-1})^+ - (Y_{t+L-1} - D_{t+L-1})^+ + Y_{t+L}$, 
and  
$Y_{t+L} = (Y_{t+L} - D_{t+L})^+ - (D_{t+L} - Y_{t+L})^+ + D_{t+L}$.

Substituting these relations, we can express the total cost $C^\pi$ as follows:
\begin{equation}\label{ProofLemma1: eq3}
\begin{aligned}
C^\pi = &\sum_{t=1}^{T-L} \left( 
\hat{c} (X_{t+L-1,1} - D_{t+L-1})^+ - \hat{c} (Y_{t+L-1} - D_{t+L-1})^+ + \hat{c} Y_{t+L} 
\right) + \\
&\sum_{t=L+1}^T \left( 
\hat{\theta} (X_{t,1} - D_t)^+ + \hat{h} (Y_t - D_t)^+ + \hat{b} (D_t - Y_t)^+ 
\right) - \hat{c} \left( (Y_T - D_T)^+ - (X_{T,1} - D_T)^+ \right).
\end{aligned}
\end{equation}

Rearranging terms, we obtain three key relations:
\begin{align}
    &\sum_{t=1}^{T-L} \hat{c} \left( (Y_{t+L}-D_{t+L})^+ - (Y_{t+L-1}-D_{t+L-1})^+ \right) - \hat{c} (Y_T-D_T)^+ = 0 \label{ProofLemma1: eq4} \\
    &\sum_{t=1}^{T-L} \hat{c} (X_{t+L-1,1}-D_{t+L-1})^+ + \hat{c} (X_{T,1}-D_T)^+ = \sum_{t=L+1}^T \hat{c} (X_{t,1}-D_t)^+ \label{ProofLemma1: eq5} \\
    &\sum_{t=1}^{T-L} \left( \hat{c} D_{t+L} - \hat{c} (D_{t+L}-Y_{t+L})^+ \right) = \sum_{t=L+1}^T \hat{c} \left( D_t - (D_t-Y_t)^+ \right) \label{ProofLemma1: eq6}
\end{align}

These relations hold because the inventory level is depleted at timestep $t = 1$ and remains so until timestep $t = L$. By combining \eqref{ProofLemma1: eq4}, \eqref{ProofLemma1: eq5}, and \eqref{ProofLemma1: eq6}, we establish the result stated in Lemma \ref{Lemma1}.
\end{proof}

\section{Proof of Lemma 2}\label{appendix:3}
\begin{proof}
From the expression for the expired quantity at timestep $t$ and the order of events, we have:
\begin{equation}\label{ProofLemma2: eq5}
O_t = \left(s - \sum_{j=t-L-m+1}^t D_j + \sum_{j=t-m-L+1}^{t-1} l_j - \sum_{j=t-m-L+1}^{t-1} O_j\right)^+ \geq \left(s - \sum_{j=t-m-L+1}^t D_j\right)^+ - \sum_{j=t-m-L+1}^{t-1} O_j.
\end{equation}

Taking the expectation of the left- and right-hand sides yields:
\begin{equation}\label{{ProofLemma2: eq6}}
\mathbb{E}\left[O_t\right] \geq \frac{\mathbb{E}\left[\left(s-\mathcal{D}_{m+L}\right)^+\right]}{m+L}.
\end{equation}

For any $t > L$, we can write:
\begin{equation}\label{ProofLemma3: eq1}
\left(\sum_{i=1}^m x_{t,i} - D_t\right)^+ \geq \left(s - \sum_{j=t-L}^t \xi_j\right)^+ - \sum_{j=t-L}^{t-1} O_j,
\end{equation}
and
\begin{equation}\label{ProofLemma3: eq2}
\left(D_t - \sum_{i=1}^m x_{t,i}\right)^+ \geq \left(\sum_{j=t-L}^t \xi_j - s\right)^+ - \sum_{j=t-L}^{t-1} l_j.
\end{equation}

This leads to:
\begin{equation}\label{{ProofLemma3: eq4}}
\begin{cases}
\mathbb{E}\left[\left(\sum_{i=1}^m x_{t,i} - D_t\right)^+\right] \geq \mathbb{E}\left[\left(s - \mathcal{D}_{L+1}\right)^+\right] + \frac{L\mathbb{E}\left[\left(s - \mathcal{D}_{m+L}\right)^+\right]}{m+L} \\
\mathbb{E}\left[\left(D_t - \sum_{i=1}^m x_{t,i}\right)^+\right] \geq \frac{\mathbb{E}\left[\left(\mathcal{D}_{L+1} - s\right)^+\right]}{L+1}
\end{cases}.
\end{equation}
\end{proof}

\section{Bounds Value Analysis}\label{appendix:4}
Table \ref{tab:results_bounds} presents the lower and upper bounds for both the OUT and PIL policies, calculated under the worst-case setting described in Section \ref{sec:section4}, along with the corresponding optimal values obtained.

Both policies exhibit similar trends. The upper bounds consistently align with or slightly exceed the optimal values, confirming their reliability. Conversely, the lower bounds remain equal to or below the optimal values, providing a conservative estimate. Interestingly, while the lower bounds are identical for both policies, the upper bound for the PIL policy is generally lower than that of the OUT policy. This observation suggests that the PIL policy offers more constrained bounds, possibly contributing to its cost-effectiveness in certain scenarios.

\begin{table}[ht!]
    \centering
    \caption{Lower (LB), upper (UB), and optimal (OPT) values for the OUT and PIL policies, with demand noise modeled as $\xi_t \sim \mathcal{N}(0, \bar{d} \times 15\%)$, where $\bar{d} = \max_t{d_t}$. The results are shown for parameter values $m = \{2, 3, 4\}$, $w = \{2, 4\}$, and $b = \{10, 50, 100, 1000\}$.}
    \label{tab:results_bounds}
    \scalebox{0.80}{
    \setlength{\tabcolsep}{4pt} 
    \renewcommand{\arraystretch}{0.5} 
    \begin{tabular}{c|c|cccc|cccc|cccc|cccc|cccc|cccc}
        \toprule
        \multicolumn{2}{c|}{} & \multicolumn{4}{c|}{OUT LB} & \multicolumn{4}{c|}{OUT OPT} & \multicolumn{4}{c|}{OUT UB} & \multicolumn{4}{c|}{PIL LB} & \multicolumn{4}{c|}{PIL OPT} & \multicolumn{4}{c}{PIL UB} \\ \midrule
        \multicolumn{1}{c|}{$m$} & \multicolumn{1}{c|}{$w$} & \multicolumn{4}{c|}{$b$} & \multicolumn{4}{c|}{$b$} & \multicolumn{4}{c|}{$b$} & \multicolumn{4}{c|}{$b$} & \multicolumn{4}{c|}{$b$} & \multicolumn{4}{c}{$b$} \\ 
        & & 10 & 50 & 100 & 1000 & 10 & 50 & 100 & 1000 & 10 & 50 & 100 & 1000 & 10 & 50 & 100 & 1000 & 10 & 50 & 100 & 1000 & 10 & 50 & 100 & 1000 \\ \midrule
        
        \multirow{2}{*}{2} & 2 & 2  & 4  & 5  & 7  & 2  & 5  & 6  & 8  & 4  & 6  & 7  & 9  & 2  & 4  & 5  & 7  & 3  & 4  & 5  & 7  & 3  & 5  & 6  & 8  \\
                           & 4 & 1  & 4  & 4  & 7  & 2  & 5  & 5  & 8  & 4  & 6  & 7  & 9  & 1  & 4  & 4  & 7  & 2  & 4  & 5  & 7  & 2  & 5  & 5  & 8  \\ \midrule
        
        \multirow{2}{*}{3} & 2 & 2  & 4  & 5  & 7  & 3  & 5  & 6  & 8  & 4  & 6  & 7  & 9  & 2  & 4  & 5  & 7  & 3  & 5  & 6  & 8  & 3  & 5  & 6  & 8  \\
                           & 4 & 1  & 4  & 5  & 7  & 3  & 5  & 6  & 8  & 4  & 6  & 7  & 9  & 1  & 4  & 5  & 7  & 3  & 5  & 6  & 8  & 3  & 5  & 6  & 8  \\ \midrule
        
        \multirow{2}{*}{4} & 2 & 2  & 4  & 5  & 7  & 3  & 5  & 6  & 8  & 4  & 6  & 7  & 9  & 2  & 4  & 5  & 7  & 3  & 6  & 6  & 8  & 3  & 5  & 6  & 8  \\
                           & 4 & 2  & 4  & 5  & 7  & 3  & 5  & 6  & 8  & 4  & 6  & 7  & 9  & 2  & 4  & 5  & 7  & 3  & 6  & 6  & 8  & 3  & 5  & 6  & 8  \\ 
        \bottomrule
    \end{tabular}
    }
\end{table}


\end{document}